\documentclass{article}
\usepackage[preprint]{neurips_2025}

\usepackage{amsmath,amsfonts,bm}

\def\eqref#1{equation~\ref{#1}}

\def\1{\bm{1}}

\DeclareMathAlphabet{\mathsfit}{\encodingdefault}{\sfdefault}{m}{sl}
\SetMathAlphabet{\mathsfit}{bold}{\encodingdefault}{\sfdefault}{bx}{n}

\newcommand{\E}{\mathbb{E}}

\newcommand{\R}{\mathbb{R}}

\newcommand{\Var}{\mathrm{Var}}

\DeclareMathOperator{\Tr}{Tr}

\usepackage{hyperref}
\usepackage{etoolbox}
\usepackage[dvipsnames]{xcolor}
\hypersetup{
  colorlinks=true,
  citecolor=ForestGreen, 
  linkcolor=black,
  urlcolor=blue,
  pdfborder={0 0 0}
}

\usepackage{url}
\usepackage{comment}
\usepackage{amsmath,amssymb,amsfonts,amsthm}
\usepackage{bbm}
\usepackage[utf8]{inputenc}
\usepackage[T1]{fontenc}
\usepackage{booktabs}
\usepackage{nicefrac}
\usepackage{microtype}
\usepackage{xcolor}
\usepackage{mdframed}
\usepackage{caption}
\usepackage{subcaption}
\usepackage{graphicx}
\usepackage{float}
\usepackage{wrapfig}
\usepackage{tikz}
\usetikzlibrary{arrows.meta,positioning}
\usepackage{algorithm}
\usepackage{algpseudocode}
\usepackage{enumitem}
\usepackage{etoolbox}
\usepackage[font=small,labelfont=bf]{caption}

\newtheorem{theorem}{Theorem}[section]

\theoremstyle{definition}

\theoremstyle{plain}

\title{Escaping Model Collapse via Synthetic Data Verification: \\ Near-term Improvements and Long-term Convergence\thanks{%
This project is supported by the AI2050 program at Schmidt Sciences (Grant G-24-66104) and Army Research Office Award W911NF-23-1-0030. 
We also thank Cong Ma from UChicago, Hongning Wang and Bo Li from Tsinghua, 
Pinyan Lu and Gavin Tang from Shanghai University of Finance and Economics for helpful discussions and suggestions.}}

\author{
  Bingji Yi\footnotemark[3]\hspace{0.5em}\footnotemark[2]
  \And
  Qiyuan Liu\footnotemark[4]\hspace{0.5em}\footnotemark[2]
  \And
  Yuwei Cheng\footnotemark[4]
  \And
  Haifeng Xu\footnotemark[5]
}

\begin{document}
\maketitle

\vspace{-2em}
{
\renewcommand{\thefootnote}{\fnsymbol{footnote}}
\footnotetext[3]{Independent Researcher. Work done while visiting UChicago CS. \texttt{yibingji@gmail.com}}
\footnotetext[4]{Department of Statistics, University of Chicago. \texttt{qiyuanliu@uchicago.edu}, \texttt{yuweicheng@uchicago.edu}}

\footnotetext[5]{Department of Computer Science, University of Chicago. \texttt{haifengxu@uchicago.edu}}
\footnotetext[2]{Equal contribution.}

\renewcommand{\thefootnote}{\arabic{footnote}}
}

\begin{abstract}

   Synthetic data has been increasingly used to train frontier generative models. However, recent studies raise key concerns that iteratively retraining a generative model on its self-generated synthetic data may keep deteriorating model performance, a phenomenon often coined \emph{model collapse}. In this paper, we investigate ways to modify the synthetic retraining process to avoid model collapse, and even possibly help reverse the trend from collapse to improvement. Our key finding is that by injecting information through an external synthetic data verifier, whether a human or  a better model, synthetic retraining will not cause model collapse.
   Specifically, we situate our theoretical analysis in the fundamental linear regression setting, showing that verifier-guided retraining can yield near-term improvements, but ultimately drives the parameter estimate to the verifier's “knowledge center” in the long run. Our theory further predicts that, unless the verifier is perfectly reliable, these early gains will plateau and may even reverse. Indeed, our experiments across linear regression, Variational Autoencoders (VAEs) trained on MNIST, and fining-tuning SmolLM2-135M on the XSUM task confirm these theoretical insights.

\end{abstract}
\vspace{-8mm} 
\begin{figure}[H]
    \centering
    \includegraphics[width=0.9\textwidth]{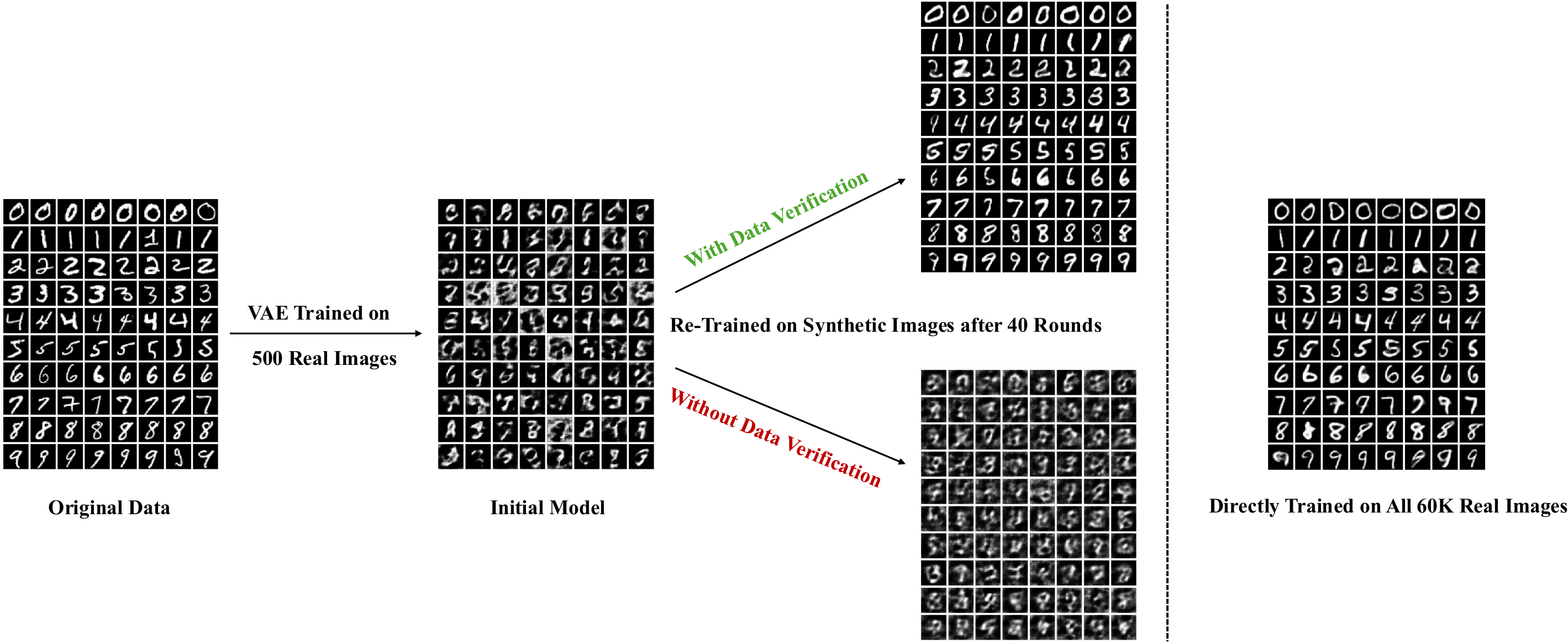} 
    \vspace{4mm}
 \caption{\label{fig:mnist-illustration}
\textbf{Iterative VAE Retraining on MNIST.}
\textbf{Left:} Original MNIST images (real data).
\textbf{Middle:} Samples from a VAE trained on 500 real images.
\textbf{Right:} Samples after 40 rounds of synthetic retraining.
The \emph{top branch} (green) uses verifier-filtered synthetic data, producing clearer and more realistic digits; 
the \emph{bottom branch} (red) retrains without verification, leading to severe degradation and mode collapse.
The final column shows a VAE trained on all 60K real images (upper bound on quality).
}

    \label{fig:samples}
\end{figure}
\vspace{-0.8em}

\section{Introduction}

The use of synthetic data has gained significant traction due to its ability to reduce data collection costs and enhance privacy protection, with applications in computer vision \citep{wood2021fake}, healthcare \citep{azizi2021can, santangelo2025good}, finance \citep{potluru2023synthetic} and recently large language models (LLMs) \citep{chen2024diversity}. A growing body of work has demonstrated that training with synthetic data can improve model performance in various applications such as image recognition \citep{hesynthetic, tremblay2018training}, image generation \citep{doersch2019sim2real,shrivastava2017learning,tian2023stablerep} and language generation \citep{gunasekar2023textbooks,guo2024deepseek,zelikman2022star}.   However, recent studies caution that recursively training models on synthetic data alone can lead to a degradation of quality, a phenomenon often coined \emph{model collapse} \citep{shumailov2024ai, dohmatob2024model, dohmatobstrong, dohmatob2024tale, alemohammad2023self, gerstgrasser2024model}. This contrast between  empirical success  and pessimistic research findings   gives rise to a natural question about what might have caused the discrepancy. 

In response to the above question, an important practical observation  is that synthetic data  are rarely used in raw form in the aforementioned applications. Instead, practitioners often apply filtering steps to remove low-quality synthetic samples before retraining. For example, in natural language generation, synthetic sentences are often   screened using grammar checkers or LLM-as-a-judge pipelines \citep{zheng2023judging,gu2024survey}; in computer vision, synthetic images can be filtered using pretrained discriminators  or human annotation \citep{hesynthetic}; in recommendation and preference learning, synthetic feedback is often validated against external heuristics or known user signals \citep{tu2025resofilter, iskander2024quality, lupidi2024source2synth, lampis2023bridging, zhang2024regurgitative}.
Common across all these approaches is the use of a knowledgeable ``discriminator'' (henceforth the \emph{verifier} ) -- whether a machine or human --   that evaluates and filters out low-quality candidate synthetic samples (i.e., those not passing the discriminator's screening).
This observation naturally  raises the following research question: 
\begin{quote} 
    \emph{Does  verifier-based filtering of synthetic data contribute to the observed empirical success of   model improvements, and does it prevent model collapse in the long run?} 
\end{quote} 

There has been a growing body of work studying the mechanisms of \emph{model collapse}, including theoretical analyzes that are often examined in the context of classic parameter estimation problems  \citep{dohmatob2024model,dohmatobstrong,gerstgrasser2024model,dey2024universality,xu2025probabilistic,sureshrate}. However, these works have all assumed the use of  synthetic data \emph{without} filtering. Only few recent works start to examine how filtered synthetic data  affects the performance of iterative retraining, but  under idealized assumptions such as access to a perfectly correct verifier  \citep{amin2025escaping}  or highly structured errors in synthetic data (i.e., i.i.d. noise added to binary labels  \citep{feng2024beyond}). A realistic and instructive framework for analyzing synthetic retraining of generative models remains still poorly understood. In this paper, we further push the boundary of this important agenda, and  examine iterative retraining under imperfect verifiers that filter out low-quality synthetic data based on their (possibly biased) knowledge. We refer to this process as    \emph{verifier-based synthetic retraining} for convenience. Specifically, we seek principled understandings about the empirically observed short-term successes of verifier-based synthetic retraining and analysis about its long-term convergence under iterative retraining.

\medskip
\noindent
\textbf{Our contributions.}     We start from  theoretical investigations and  analyze verifier-based synthetic retraining with \emph{verified} synthetic data on the foundational linear regression model---a canonical  setting that has become central to the study of model collapse (e.g., \citep{dohmatob2024model,dohmatobstrong,gerstgrasser2024model,zhusynthesize,garg2025preventing}).\footnote{The theoretical analysis of \citep{feng2024beyond} is also based on linear models, specifically, linear classifiers.} We then verify our theoretical insights through thorough empirical studies in real-world generative settings.  Our main contributions are summarized as follows:

\begin{itemize}
    \item \emph{Does verified synthetic data improve retraining and, if so, under what conditions?} 
     We show that it indeed can, provided the right conditions are met.
    Through a new form of bias-variance trade-off under data filtering,  we characterize the regimes in which verifier-based synthetic retraining   leads to strict model improvement, rather than  degradation, in   the short term (Theorem \ref{thm:linear:one_step}). The conditions we identify highlight the mixed effect of synthetic sample size, the verifier's bias  and  selectivity  during filtering, yielding practical insights regarding   \emph{when}  verification of synthetic data is beneficial.
    
    \item \emph{Can the retraining improvement in the short term be sustained in the long run?}  Our theoretical analysis shows that the answer is   \emph{No} -- unless the verifier has no bias.
    Formally, we show that verifier-based synthetic retraining  will converge to the verifier's knowledge center in the long term (Theorem \ref{thm:linear:long_term}).
    This result reveals how the verifier's quality affects  the asymptotic dynamics of iterative retraining.
    \textcolor{black}{Notably, while verifier selectivity influences short-term performance, it does not change the long-run converging point, though it does affect its convergence rate.}
    
    \item \emph{Do the insights from the above theoretical analysis apply to real generative models empirically?} Indeed, both simulations of linear regression settings   and  training Variational Autoencoders (VAEs) on the real-world MNIST data demonstrate that our theoretical insights align with observed training dynamics in these settings. For example, Figure \ref{fig:mnist-illustration} shows how the generative model of VAE   can start from a  poor model trained on a small number of 500 images and gradually improve from being retrained on filtered synthetic imaged generated by itself, until after 40 iterations of retraining generating sharp images that are visually comparable to a VAE trained directly on the entire MNIST dataset. More quantitative analysis of our empirical studies can be found in Section \ref{sec:experiment}.     
\end{itemize}

\subsection{Related Work}

\paragraph{Understanding and mitigating model collapse.}
Recent work shows that heavy reliance on synthetic data in iterative training can cause \emph{model collapse}—the degradation of performance when a model is repeatedly retrained on its own synthetic outputs (possibly mixed with real data).\footnote{There is no widely agreed-upon formal definition; see \citep{schaeffer2025position} for discussion.} Empirical evidence supports this phenomenon: \citet{shumailov2024ai} show that recursive training on unfiltered synthetic data induces distribution shift and mode collapse, while \citet{dohmatobstrong} find that even small synthetic proportions can harm performance. In linear settings, \citet{dohmatob2024model} analyze collapse mechanisms explicitly, and \citet{dohmatob2024tale} connect degradation to altered neural scaling laws.

\textcolor{black}{To mitigate collapse, prior work broadly explores three strategies.
First, accumulating data or gradually increasing the synthetic dataset size across iterations can suppress noise and bound errors \citep{gerstgrasser2024model, dey2024universality, xu2025probabilistic, kazdancollapse, barzilai2025models}.
Second, mixing synthetic data with real data stabilizes retraining \citep{bertrand2024stability, fu2024towards, futheoretical}, as performance progressively degrades without sufficient fresh real data \citep{alemohammad2023self}.
Recent studies have even derived optimal mixing ratios to maximize this stabilizing effect \citep{he2025golden, garg2025preventing}.
Finally, algorithmic interventions, such as the token re-sampling procedures proposed by \citet{zhusynthesize}, offer alternative pathways to avoid collapse.}

Unlike prior work that relies on unfiltered synthetic data, our framework incorporates an external verifier to remove low-quality samples. Such verifiers may be human annotators or stronger teacher models. Filtering is widely used in iterative retraining and has shown empirical success in preventing model degradation and even improving performance \citep{hesynthetic,tian2023stablerep,guo2024deepseek,zelikman2022star,zhang2024regurgitative, lampis2023bridging, haluptzoklanguage, patwa2024enhancing}. Motivated by this, we develop a principled understanding of when improvement is possible—namely, whether a generative model can leverage the verifier’s feedback, embedded in the selected synthetic subset, to achieve sustained gains.

\vspace{-2mm}
\paragraph{Filtering and selecting synthetic data.}
\textcolor{black}{While a rich line of empirical work demonstrates that these filtering strategies can improve model performance,
theoretical understanding about iterative retraining with filtered synthetic data remains largely unexplored, with only a few recent exceptions.}
\textcolor{black}{\citet{amin2025escaping} assume a strong, reliable quality function and focus on how an external labeler aids learning under this fixed filtering mechanism. 
\citet{feng2024beyond} study a classification problem and identify a sharp phase transition.
However, modeling synthetic data merely as noisy labels abstracts away the structural dependencies between features and labels inherent to true generative processes.
}
Finally, \citet{ferbach2024self,weiself} considers a conceptually similar problem of learning  a discrete preference distribution from human feedback by using humans' preferred choices as a filtering strategy.
\textcolor{black}{In their population-level analysis, curating synthetic data via an external reward function forces the model distribution to converge to the highest-reward level set, maximizing expected reward but ultimately collapsing in diversity.}

Similar to many of the aforementioned studies above, our theoretical analysis also focuses on linear models \citep{feng2024beyond,dohmatobstrong,gerstgrasser2024model}.
However, our model allows inaccuracy of the verifier in terms of both bias and variance.
Errors in the synthetic data primarily stem from the inaccuracy of the generative model itself \textcolor{black}{rather than exogenous noise}.
We show that model's short-term performance varies smoothly with the verifier's bias, selectivity, and  size of synthetic data, rather than exhibiting a sharp phrase transition from complete failure to perfect accuracy as in \citet{feng2024beyond}.
In the long run, the model's performance converges to the verifier's knowledge center whereas verifier's selectivity only affects convergence speed.
Our results bridge short-term and long-term perspectives of iterative retraining, illustrating  how varied verifier qualities give rise to distinct performances of iterative retraining.

\paragraph{Comparison with reward maximization frameworks.}
\label{app:reward_comparison}
While sharing the conceptual similarity of evaluating generated data via an external feedback mechanism, our modeling approach substantially differs from reward maximization frameworks in various aspects, including preference matching \citep{ferbach2024self,weiself} and recent Reinforcement Learning with Verified Rewards (RLVR)\citep{guo2024deepseek,yu2025dapo}. 
Theoretically, these methods frame the problem as policy optimization, where the objective is to maximize a provided reward signal. While highly effective for alignment, the definition of a ``good model'' is tied to the specific reward formulation, which may not correspond to recovering the true data-generating distribution.
Practically, reward optimization relies on assigning scalar rewards or pairwise comparison signals \citep{ouyang2022training}, which are often difficult and noisy to define. For instance, evaluating image quality or open-ended language generation with a single numerical reward is inherently subjective. While recent methods like RLVR avoid this issue by restricting themselves to domains with clearly verifiable rewards \citep{guo2025deepseek,wu2025invisible,yu2025dapo}, many important training settings lack such reliable reward functions.
In contrast, we study the widely used ``generate-verify-retrain'' paradigm, which utilizes binary accept/reject filtering. 
Theoretically, our framework defines a ``good model'' at the parameter level, explicitly modeling the relationship between the verifier's filtering rule and the ground truth. 
By formalizing this link, we can directly analyze model performance during iterative retraining, even for an imperfect or biased verifier.
Practically, this filtering mechanism is less noisy, highly stable, and serves as a core scalable primitive in modern LLM pipelines like DeepSeek-Coder \citep{guo2025deepseek}.

\section{Modeling Verifier-Based Synthetic Retraining: the Linear Regression Case}\label{sec:linear}

In this section, we formalize our model of iterative retraining with verified synthetic data, coined \emph{verifier-based synthetic retraining} for convenience. Following recent works in this space \citep{dohmatob2024model,gerstgrasser2024model,garg2025preventing,zhusynthesize}, we focus on the foundational linear regression setting where the objective is to estimate a high-dimensional coefficient vector $\theta^\star$ in the following linear model
\[
y = x^{\top} \theta^\star + \xi,
\]
where $\xi \sim \mathcal{N}(0, \sigma^2)$, $x \in \mathbb{R}^p$, and $\theta^\star \in \mathbb{R}^p$ is the unknown parameter of interest. We use the standard Mean Squared Error (MSE), i.e., $\text{MSE}(\hat{\theta}) = \E_{\xi} \|\hat{\theta} - \theta^\star\|^2$, to evaluate estimators.

\paragraph{Modeling the verifier and data filtering rule.}  Suppose we have access to a verifier that possesses prior knowledge of $\theta^\star$, modeled by a knowledge set. Specifically,  the verifier's knowledge is described by a spherical ball:
\[
B_r(\theta_c) := \bigl\{\, \theta \in \mathbb{R}^p : \|\theta - \theta_c\| \leq r \,\bigr\},
\]
with fixed center $\theta_c$ and radius $r$.  We assume this knowledge set indeed contains the true parameter, i.e.,  $\theta^\star \in B_r(\theta_c)$, but  the true parameter $\theta^\star$ is unknown.  The verifier does not reveal $\theta_c$ or $r$ directly (see modeling motivations below). Instead, it  only provides binary feedback indicating whether a given (real or synthetic) data point $(x_i, y_i)$ is consistent with the knowledge $\theta^\star \in B_r(\theta_c)$ or not. Specifically, the verifier outputs \emph{Yes} if
\begin{align}
|y_i - x_i^{\top} \theta_c| \leq r \|x_i\| + \sigma_c, \label{equ:linear:verify_condition}
\end{align}
and \emph{No} otherwise. Here $\sigma_c$ is a  constant related to the verifier's capability.  This \emph{filtering rule} is motivated by the following bound on expected errors: 
$
\E\bigl[\,|y_i - x_i^{\top}\theta_c|\,\bigr] = \E\bigl[\,|x_i^{\top}(\theta^\star - \theta_c) + \xi_i|\,\bigr] \leq r\|x_i\| + \E|\xi_i| = r\|x_i\| + \sqrt{\tfrac{2}{\pi}}\sigma.
$
Since the true $\sigma$ might be unknown in practice, $\sigma_c$ serves as an estimate of the true $\sigma$. 

We refer to $\Delta = \|\theta^\star - \theta_c\|$ as the \emph{bias} of the verifier, whereas  $r$ captures the \emph{selectivity} of the verifier --  the smaller $r$ is, less likely the verifier  accepts a data point $(x_i, y_i)$.  The verifier only needs to  provide Yes/No answers based on the above selection rule in   \eqref{equ:linear:verify_condition}, but does not  needs    to know the parameter  $\theta_c, r$ of the knowledge set.   The motivation of this modeling primarily comes from practice, as explained below.

\paragraph{Motivation of binary feedback from verifiers.}    
We adopt the  binary feedback  from verifiers mainly for practical reasons.   
In practice, eliciting simple yes/no feedback is far less noisy and more cost-effective than asking verifiers to directly specify $\theta_c$ or $r$. Indeed, in real applications verifiers may not even know these quantities explicitly, which would correspond to model parameters if the verifier is a stronger teacher model or how the human reasons if the verifier is a human.  This model choice   is also aligned with the  widely adopted comparison-based feedback in reinforcement learning from human feedback (RLHF) \citep{ouyang2022training}. Such binary feedback has become a standard approach in preference alignment for large language models, where LLM raters and human evaluators provide pairwise or accept/reject judgments that effectively guide learning at scale \citep{wettigqurating}. Although simple, our theory and empirical evaluations both show that this \emph{single bit} of information for each sample can successfully be injected into the retraining process to improve models.

\paragraph{Synthetic Retraining with Verifier-based Filtering}

We begin with an initial set of real data $(X^0, Y^0)$, where $X^0 \in \mathbb{R}^{n_0 \times p}$ and $Y^0 \in \mathbb{R}^{n_0}$.  
The initial estimator $\hat{\theta}^0$ is obtained via Ordinary Least Squares (OLS) \footnote{For ease of presentation, we assume Rank$(X^0)=p$. If Rank$(X^0)<p$, all our results are equally applicable by working in the subspace of $X^0$. }
$
\hat{\theta}^0 = ({X^{0}}^\top X^0)^{-1} {X^0}^\top Y^0. \label{equ:linear:theta0}
$
We then proceed with iterative synthetic retraining via the \emph{generate--verify--retrain} procedure outlined in Figure~\ref{fig:pipeline}, the rigorous retraining Scheme~\ref{scheme:linear_retraining} and Algorithm~\ref{alg:linear_retraining} are provided in Appendix~\ref{app:linear}.

\begin{figure}[h]
\centering
\begin{tikzpicture}[
  node distance=14mm, >=Stealth,
  midbox/.style={rectangle, rounded corners, draw, align=center,
                 minimum height=8mm, text width=20mm, inner sep=2mm, font=\footnotesize},
  endbox/.style={rectangle, rounded corners, draw, align=center,
                 minimum height=8mm, text width=10mm, inner sep=2mm, font=\footnotesize}
]
\footnotesize
\node[endbox] (theta) {$\hat{\theta}^k$};
\node[midbox, right=of theta] (xy) {$(X^{k+1},\, Y^{k+1})$};
\node[midbox, right=of xy] (xyv) {$({X^{k+1}}',\, {Y^{k+1}}')$};
\node[endbox, right=of xyv] (theta2) {$\hat{\theta}^{k+1}$};

\draw[->] (theta) -- node[midway, above] {Generate} (xy);
\draw[->] (xy) -- node[midway, above] {Verify} (xyv);
\draw[->] (xyv) -- node[midway, above] {Retrain} (theta2);
\end{tikzpicture}
\caption{Generate-Verify-Retrain pipeline.}
\label{fig:pipeline}
\end{figure}

Since learning proceeds through the conditional $Y^k \mid X^k$, synthetic retraining requires specifying the covariate design $X^k$; labels $Y^k$ are then generated conditionally via the model under verifier constraints.
In principle, one could construct $X^k$ arbitrarily; however, for mathematical clarity, below we describe a targeted though arguably natural design. 
In particular, we choose to align the synthetic covariates with a fixed orthonormal set $\{v_1,\dots,v_p\}$ and construct $X^k$ in a block-structured form by repeating each $v_j^\top$ as rows:
\label{equ:linear:block_design}
\begin{align}
X^k = (\,
\underbrace{v_1, \ldots}_{\text{copies of } v_1},\;
\underbrace{v_2, \ldots}_{\text{copies of } v_2},\;
\ldots,\;
\underbrace{v_p, \ldots}_{\text{copies of } v_p}
\,)^\top.
\end{align}
After verifier filtering, each orthogonal direction $v_j$ retains exactly $n_k$ samples with 
$
n_0 \;\leq\; pn_1 \;\leq\; pn_2 \;\leq\; \cdots.
$

Notably, the estimation of parameter $\hat{\theta}^{k+1}$ using only synthetic data from the model with $\hat{\theta}^{k}$, though with filtering, leads to a Markovian transition   $\hat{\theta}^k \mapsto \hat{\theta}^{k+1}$. The above block design essentially helps ``diagonalize''  the transition operator $\hat{\theta}^k \mapsto \hat{\theta}^{k+1}$.  
The conceptual benefit of this covariance design choice is that we remove the rotational variability that arbitrary designs would introduce across iterations and decouple the dynamics along orthogonal directions.
In practice,  this design mirrors curating data along approximately orthogonal latent spaces or topics (e.g., topical axes like politics, sports, mathematics).
However, the choice of covariant $X^k$ is not unique:   alternatives (e.g., canonical basis, isotropic random directions) can yield similar qualitative conclusions, with potentially different constants or rates.
We expect our theoretical insights to generalize to any reasonable design of the covariant $X^k$, though the rigorous proofs may be less tractable for some designs.

\section{On the Near-Term Improvement under Synthetic Retraining} 

This section investigates the verifier's role in synthetic retraining: \emph{does it help, when does it help, and why does it help?} 
We focus on one round and show that verifier-guided retraining can improve performance under mild assumptions.
The underlying mechanism is a verifier-induced bias-variance trade-off: filtering synthetic data \emph{reduces variance} but may \emph{introduce bias}. 

\subsection{Source of Improvement:  Bias--Variance Trade-off with Verifier's Bias }  

To address the question of \emph{when and why} verifier-based synthetic retraining improves estimation, we analyze the mean squared error (MSE) of the one-step estimator $\hat{\theta}^1$ in estimating the true regression coefficient $\theta^\star$. The MSE admits the following decomposition:

\begin{equation}\label{equ:regression-risk-decomp}
    \mathbb{E}\|\hat{\theta}^1 - \theta^\star\|^2
    = \mathbb{E}_{\hat{\theta}^0}\!\left[\, \Tr\!\big( \Var(\hat{\theta}^1 \mid \hat{\theta}^0) \big) \,\right]
    + \mathbb{E}_{\hat{\theta}^0}\!\left\|\, 
        \mathbb{E}\!\left[\hat{\theta}^1 \mid \hat{\theta}^0\right] - \theta^\star 
    \right\|^2.
\end{equation}

The first term in \eqref{equ:regression-risk-decomp} is the \textbf{synthetic variance}: it captures additional estimation noise from the randomness in synthetic data generation. This variance decreases at rate $1/n_1$ with the synthetic sample size $n_1$, but is unaffected by the real sample size $n_0$. Hence, with abundant synthetic data, this term becomes negligible.

The second term is the \textbf{verification error}, which measures the deviation of the conditional mean estimator $\mathbb{E}(\hat{\theta}^1 \mid \hat{\theta}^0)$ from $\theta^\star$. This error depends both on the accuracy of the verifier (i.e., its potential bias) and the quality of the initial estimator $\hat{\theta}^0$, which improves with larger $n_0$.

To further disentangle the verification error, we decompose it as
\begin{equation}
    \mathbb{E}_{\hat{\theta}^0}\!\left\|\, 
        \mathbb{E}\!\left[\hat{\theta}^1 \mid \hat{\theta}^0\right] - \theta^\star 
    \right\|^2
    = \Tr\!\left( \Var\!\left( \mathbb{E}\!\left[\hat{\theta}^1 \mid \hat{\theta}^0\right] \right) \right)
    + \|\mathbb{E}[\hat{\theta}^{\,1}] - \theta^\star\|^2. 
    \label{equ:regression-risk-decomp2}
\end{equation}
Here, the first term is the \textbf{verification variance}, reflecting variance reduction achieved by discarding inconsistent synthetic samples, while the second is the \textbf{verification bias}, capturing systematic deviation introduced by verifier bias.

Putting these together, the full decomposition is
\begin{equation}\label{equ:regression-risk-decomp3}
    \mathbb{E}\|\hat{\theta}^{\,1}-\theta^\star\|^2
    = \underbrace{\mathbb{E}_{\hat{\theta}^0}\!\left[\, \Tr\!\big( \Var(\hat{\theta}^1 \mid \hat{\theta}^0) \big) \,\right]}_{\text{Synthetic Variance}} 
    + \underbrace{\Tr\!\left( \Var\!\left( \mathbb{E}\!\left[\hat{\theta}^1 \mid \hat{\theta}^0\right] \right) \right)}_{\text{Verification Variance}} 
    + \underbrace{\|\mathbb{E}[\hat{\theta}^{\,1}] - \theta^\star\|^2}_{\text{Verification Bias}}.
\end{equation}

This decomposition highlights the central trade-off: verifier filtering \emph{reduces variance} but may \emph{introduce bias}. Verified synthetic data leads to improvement precisely when the variance reduction outweighs the bias introduced. In particular, when the verifier is sufficiently accurate and the synthetic sample size $n_1$ is large, the MSE of $\hat{\theta}^1$ can be strictly smaller than that of the real-data estimator $\hat{\theta}^0$. 

\subsection{Characterizing Improvement in One-Round Retraining} 

\textcolor{black}{
The following theorem rigorously quantifies this trade-off for the linear regression model, demonstrating exactly \emph{when} the one-step estimator $\hat{\theta}^1$ improves or degrades upon the initial baseline.
By characterizing the MSE of $\hat{\theta}^1$, it reveals how synthetic sample size, verifier bias, and verifier selectivity determine the final outcome.
}

\begin{theorem}\label{thm:linear:one_step}
Let $\{\mu_j\}_{j=1}^p$ denote the singular values of $X^0$, 
and assume each of them satisfies $\mu_j = \Omega(\sqrt{n_0})$.\footnote{That is, each dimension is well-represented in the original data. This holds easily when, e.g., the feature data is drawn i.i.d. from a full-rank distribution.}
Then there exist constants $m_{1,j},m_{3,j} \in \R$ and $m_{2,j}\in (0,1)$ for $j=1, \ldots, p$,  as well as a constant $L>0$ such that:
\begin{equation}
    \frac{1}{\sigma^2}\text{MSE}(\hat{\theta}^1) = \sum_{j=1}^{p} \bigg( \underbrace{\frac{m_{2,j}}{n_1}}_{\text{\scriptsize Synthetic Variance}} + \underbrace{m^2_{1,j}+\frac{m_{1,j}m_{3,j}+m^2_{2,j}}{\mu^2_j}}_{\text{\scriptsize Verification Error}} \bigg) + \mathcal{O}\left( n^{-4/3}_0 \right)\label{equ:linear:mse_one_step}
\end{equation}
holds with probability at least \( 1 - p\exp\left(-Ln_0^{1/3}\right) \), where \( n_1 \) denotes the post-verification sample size.
 
\end{theorem}

\textcolor{black}{While the explicit forms of the constants are deferred to Appendix~\ref{app:linear}, their roles are highly intuitive: $m_{1,j}$ and $m_{3,j}$ capture the directional bias between the verifier's knowledge center $\theta_c$ and the ground truth $\theta^*$ along the $j$-th singular direction (vanishing if $\theta_c = \theta^*$), while $m_{2,j} < 1$  quantifies the variance reduction along that direction.
Theorem~\ref{thm:linear:one_step} mathematically guarantees when verifier-guided retraining improves the model.
Since the scaled baseline error is $\frac{1}{\sigma^2}\text{MSE}(\hat{\theta}^0) = \sum_{j=1}^p \mu_j^{-2}$, we can directly compare it to Equation~\ref{equ:linear:mse_one_step}.
When the verifier is highly accurate ($m_{1,j}, m_{3,j} \approx 0$), the verification error term becomes dominated by $\sum_{j=1}^p m_{2,j}^2 / \mu_j^2$. 
Because $m_{2,j} < 1$, this verification error is strictly smaller than the real-data error.
Thus, whenever the verified synthetic sample size $n_1$ is sufficiently large to drive the synthetic variance down, $\text{MSE}(\hat{\theta}^1)$ strictly improves upon the baseline.
}

This result highlights why verifier-based retraining is practically useful: in modern machine learning systems where real data collection is costly but generative models or simulators are available, a moderately accurate verifier can filter synthetic samples to effectively amplify limited real-world evidence and substantially reduce estimation error.
\textcolor{black}{Conceptually, this offers a sharp departure from classical model collapse literature, which typically models iterative synthetic data purely as a variance-inflating noise source \citep{shumailov2024ai,alemohammad2023self,dohmatob2024model}. Here, we prove that \emph{verification transforms synthetic data into a variance-reducing resource}, provided the verifier's bias is sufficiently controlled.}
As we will demonstrate empirically in Section~\ref{sec:experiment}, this mechanism is not confined to the linear model; it manifests clearly in complex models such as VAEs and LLMs.

\section{Iterative Retraining  Converges to the Verifier's Knowledge Center}

\textcolor{black}{
Having established that a single round of verifier-based retraining can improve estimation through a bias--variance trade-off, a natural question arises:
\emph{can this improvement be sustained over multiple rounds, and what is the eventual outcome?}
To contextualize our long-term analysis within the broader literature on model collapse, we first formalize these widely discussed empirical phenomena within our linear regression setting.
Specifically, we define \textbf{Model Degradation/Collapse} as $\limsup_{k \to \infty} \text{MSE}(\hat{\theta}^k)> \text{MSE}(\hat{\theta}^0)$,
and \textbf{Model Improvement} as $\limsup_{k \to \infty} \text{MSE}(\hat{\theta}^k) < \text{MSE}(\hat{\theta}^0)$.
}

Our key finding is that both behaviors can occur in long-term iterative retraining. 
The outcome depends critically on three factors: the growth rate of synthetic data, the verifier's bias, and the verifier's capability (i.e., its ability to reduce variance). 
Over time, iterative retraining injects increasingly more verifier knowledge into the estimator, while the contribution from the original data gradually decays. 
As a result, the verifier and the generative model family eventually dominate the limit behavior, driving the estimator $\hat{\theta}^k$ toward a fixed point, which corresponds to the verifier's knowledge center $\theta_c$. 

This dynamic gives rise to three distinct phases of long-term behavior: \textbf{(1) Unbiased verifier:} If the verifier is unbiased (i.e., $\theta_c= \theta^\star$), iterative retraining yields continuous improvement and the estimator converges to the true parameter.
\textbf{(2) Mildly biased verifier:} With small bias, iterative retraining can improve performance in the short term by reducing variance, but performance eventually plateaus or deteriorates as verifier bias accumulates.
\textbf{(3) Strongly biased verifier:} With large bias, iterative retraining leads to degradation and may even cause collapse in the limit.
\textcolor{black}{Among these, the mildly biased case is the most practically relevant.
It highlights a cautionary message:
while synthetic retraining can initially boost accuracy, a perfectly unbiased verifier is unrealistic; consequently, this inherent bias will ultimately prevent sustained improvement.}

Formally, the following theorem characterizes the long-term behavior of the estimator $\hat{\theta}^k$ in linear regression under   verifier-based synthetic retraining.

\begin{theorem}\label{thm:linear:long_term}  
There exist synthetic retraining processes (e.g., Algorithm~\ref{alg:linear_retraining}) and a constant $0<\rho < 1$ such that:
\vspace{-3mm}
\begin{align}
\mathbb{E} \|\hat{\theta}^k-\theta_c\|^2 \leq \rho^{2k} \mathbb{E}\|\hat{\theta}^0-\theta_c\|^2 + p\sigma^2\sum_{j=0}^{k-1}\frac{\rho^{2(k-j)-1}}{n_j}. \label{equ:linear:mse_long_term}
\end{align}

where $n_1 \leq n_2 \leq \cdots \leq n_k \cdots$ denote the number of verified synthetic samples per direction at each iteration.
In particular, if $\lim_{k \to \infty} n_k = \infty$, then $\lim_{k \to \infty} \mathbb{E}\|\hat{\theta}^k - \theta_c\|^2 = 0$.
\end{theorem}

The proof of Theorem~\ref{thm:linear:long_term} is provided in Appendix~\ref{app:linear}, utilizing concentration bounds and supermartingale inequalities to establish convergence.
Here we focus on the main intuition and highlight the key novelty of our analysis.  

The central observation to establish Theorem~\ref{thm:linear:long_term} is that the iterative retraining procedure (see Figure~\ref{fig:pipeline}) induces a \emph{Markov process}: the next state $\hat{\theta}^{k+1}$ depends only on the current state $\hat{\theta}^k$.  
Formally, the update can be expressed as
\begin{align}
\hat{\theta}^{k+1} = T(\hat{\theta}^k) + \eta_{k+1}, \label{equ:linear:intuition}
\end{align}
where $T(\cdot)$ is a deterministic mapping determined by verifier filtering, and $\eta_{k+1}$ is a sub-Gaussian noise term due to the randomness of synthetic samples at iteration $k+1$.  
Crucially, we show that $T(\cdot)$ is a \emph{contraction mapping}, and that the variance of the noise decays at the rate $\Var(\eta_{k+1}) \asymp 1/n_{k+1}$. This perspective allows us to view \eqref{equ:linear:intuition} as a discretized stochastic differential equation (SDE).  
As $n_k \to \infty$, the noise term vanishes and the dynamics are dominated by the deterministic contraction $T(\hat{\theta}^k)$, which drives the recursion toward its fixed point---the verifier's knowledge center $\theta_c$.  
The presence of the verifier is therefore \emph{essential}: it is precisely what transforms the update rule into a contraction, guaranteeing convergence.

By contrast, in prior work on model collapse without a verifier (e.g., \citet{gerstgrasser2024model, xu2025probabilistic}), the update reduces to the identity mapping.  
In that case, increasing the synthetic sample size can suppress noise accumulation and ensure bounded error (i.e., $\text{MSE}(\hat{\theta}^k) < \infty$), but there is no contraction and hence no convergence or sustained improvement. 
\textcolor{black}{The knowledge extracted from the verifier is precisely what elevates $T(\cdot)$ beyond an identity mapping.}
Our analysis is the first to formally show that the verifier fundamentally alters the long-term dynamics: it continuously injects knowledge, iteration by iteration, so that the estimator moves closer to $\theta_c$ over time.  

This contribution also clarifies a common misconception: even with a perfect verifier ($\theta_c = \theta^\star$) and infinitely many synthetic samples in one iteration, convergence cannot occur in a single step.  
As shown in Theorem~\ref{thm:linear:one_step}, while infinite samples remove the synthetic variance term, the verification error term persists.  
Thus, convergence requires the \emph{iterative} action of the verifier, which gradually aligns the estimator with the truth.


\section{Experiments}\label{sec:experiment}

In this section, we validate our theoretical predictions across three settings: a \emph{linear regression simulation} that is consistent with our analytical assumptions,
\textcolor{black}{alongside \emph{Variational Autoencoders (VAEs) on MNIST} and fine-tuning a pretrained \texttt{SmolLM2-135M} \citep{allal2025smollm2smolgoesbig} on a large-scale news summarization task,
which together illustrate practical behavior under iterative retraining and filtering.}
Across all settings, the empirical results closely align with our theoretical predictions. Experimental code used to generate the results is publicly
available at \url{https://github.com/liuqiyuanhhh/Verified-Synthetic-Data}.

\begin{figure}[t] 
    \centering
    \begin{subfigure}[b]{0.43\textwidth}
        \centering
        \includegraphics[width=\linewidth]{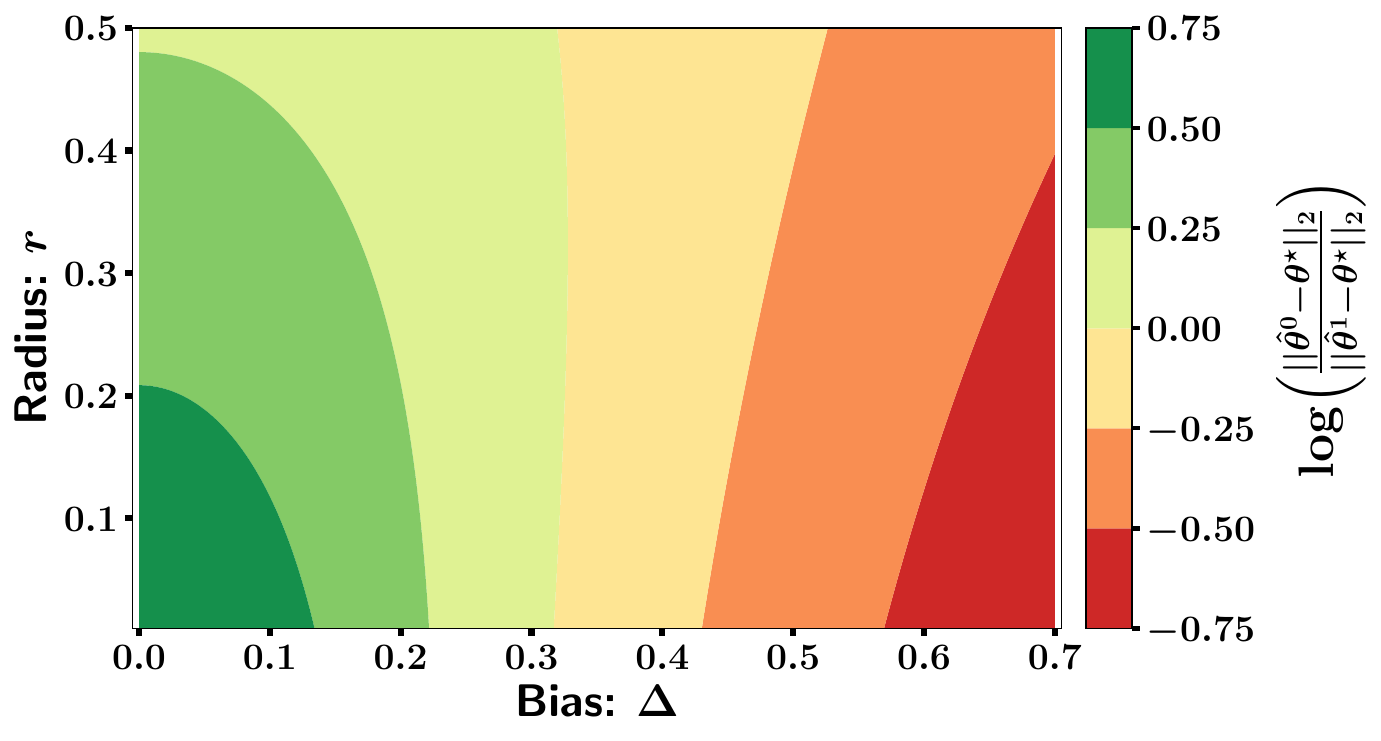}
        \caption{One-step: Theoretical Loss Landscape}
        \label{fig:linear_one_round_theory}
    \end{subfigure}
    \hfill
    \begin{subfigure}[b]{0.43\textwidth}
        \centering
        \includegraphics[width=\linewidth]{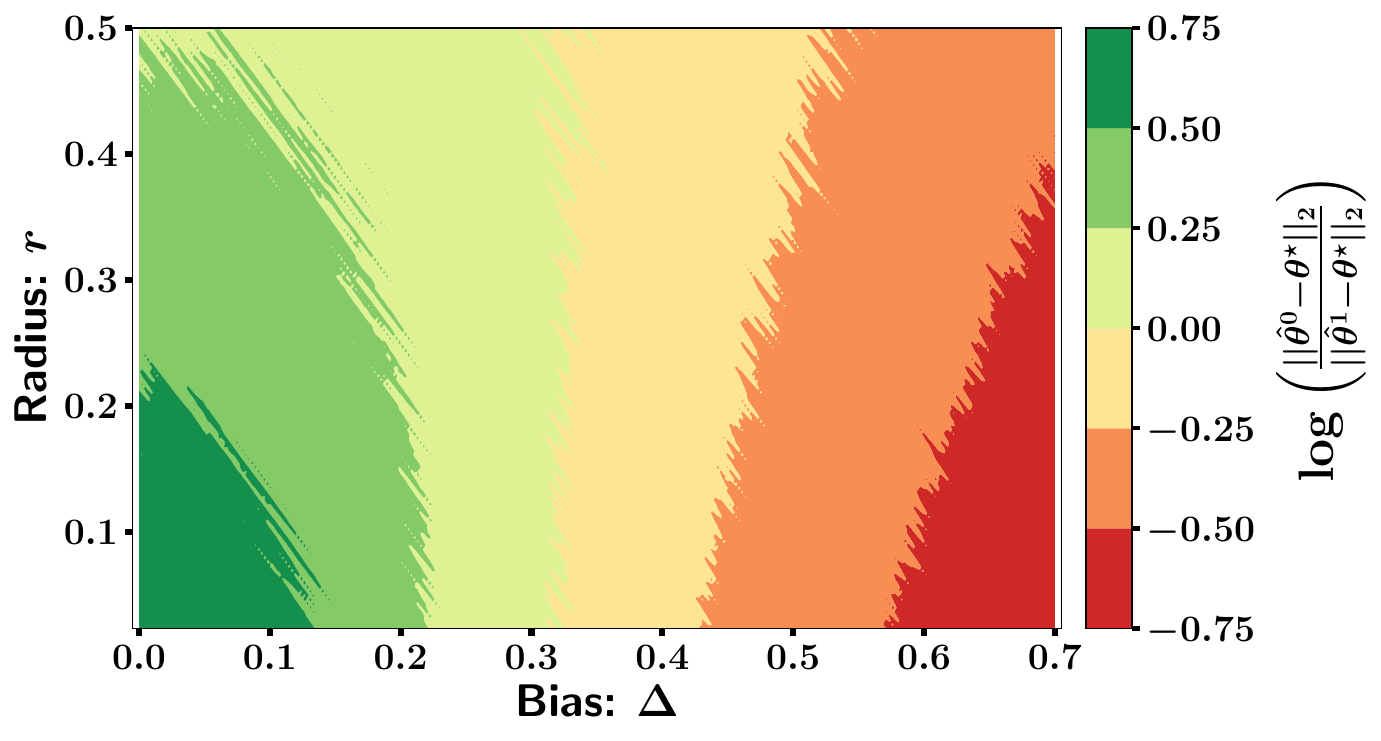}
        \caption{One-step: Empirical Loss Landscape}
        \label{fig:linear_one_round_empirical}
    \end{subfigure}

    \vspace{0mm} 

    \begin{subfigure}[b]{0.45\textwidth}
        \centering
        \includegraphics[width=\linewidth]{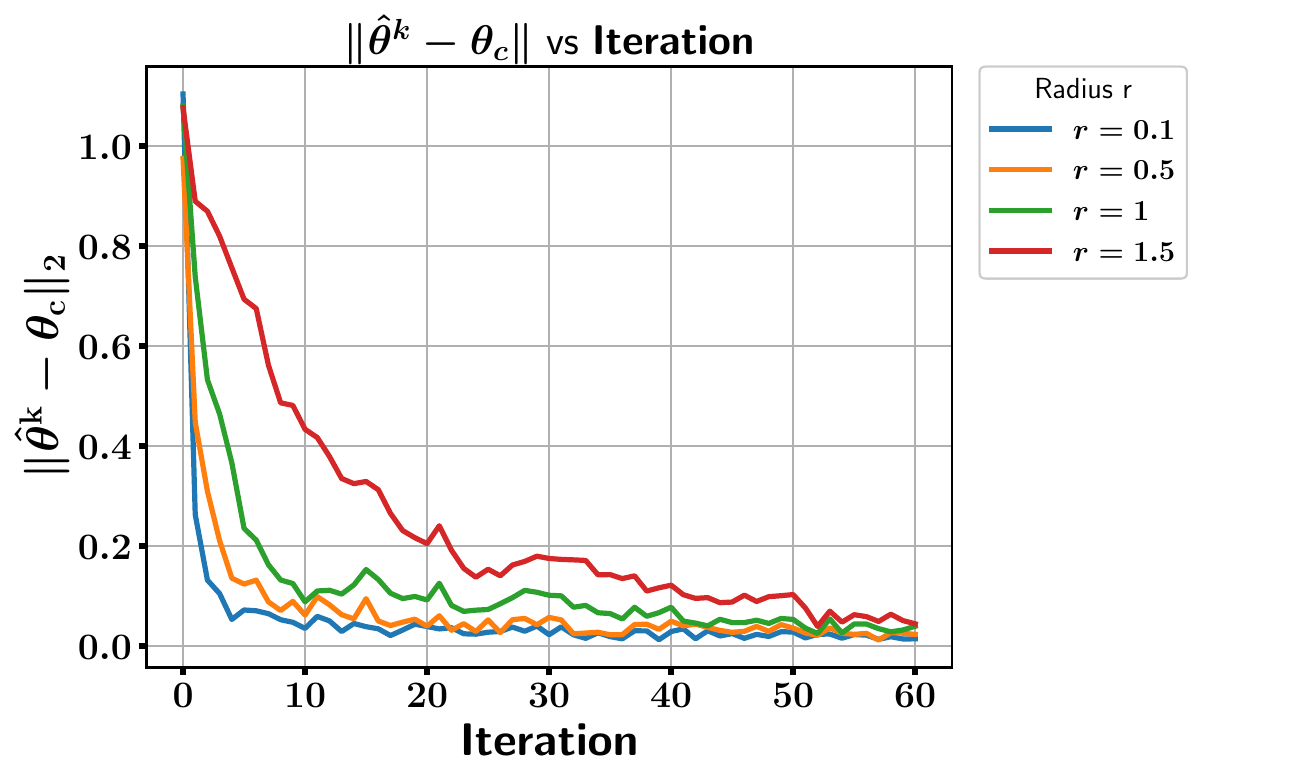}
        \caption{Iterative: Verifier with bias}
        \label{fig:linear1}
    \end{subfigure}
    \hfill
    \begin{subfigure}[b]{0.45\textwidth}
        \centering
        \includegraphics[width=\linewidth]{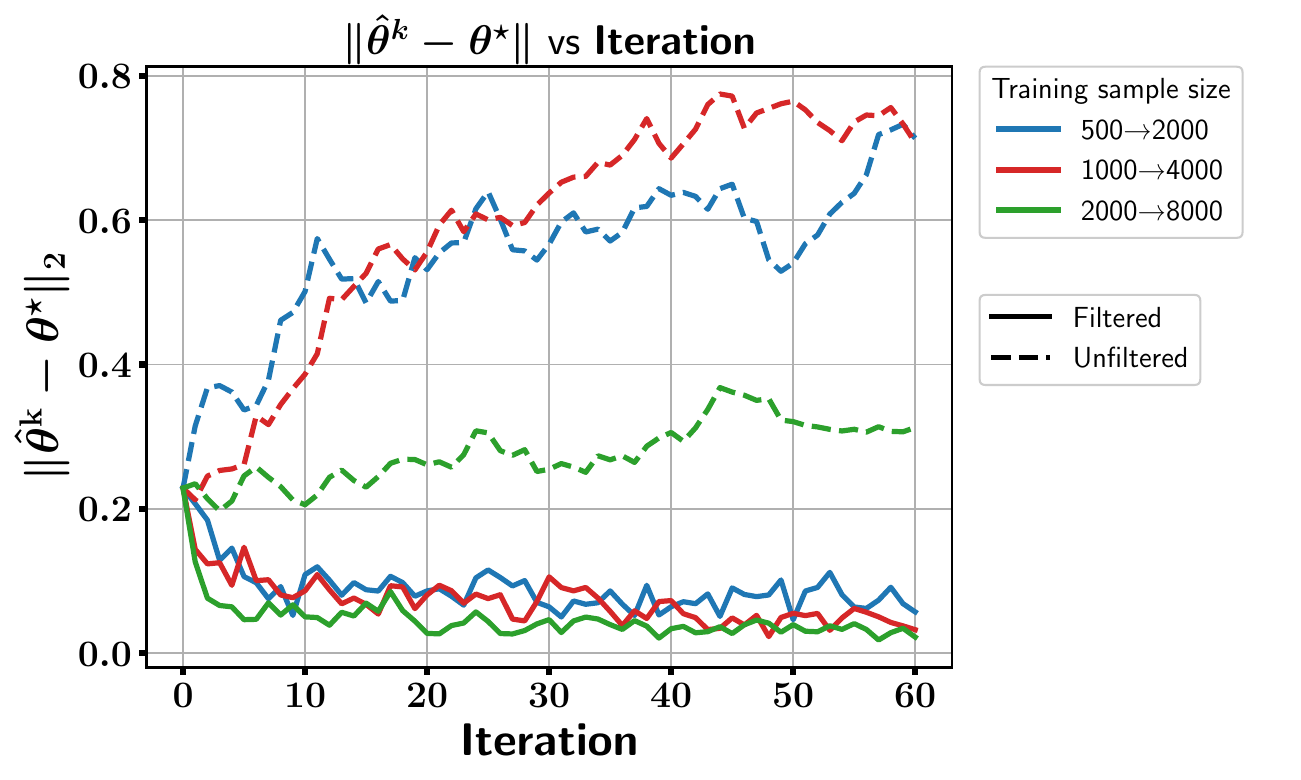}
        \caption{Iterative: Verifier without bias}
        \label{fig:no_bias_linear_regression}
    \end{subfigure}
    \caption{
        \textbf{Top:} Error changes of the one-step  retraining estimator $\hat{\theta}^1$  versus  estimator $\hat{\theta}^0$ only using  original real data, measured by $\log(\frac{||\hat{\theta}^0 - \theta^\star||}{||\hat{\theta}^1 - \theta^\star||})$: theory's prediction (a) and empirical comparisons (b). 
        \textbf{Bottom:} Iterative synthetic retraining over 60 rounds with a biased verifier ($\Delta=1$) and an unbiased verifier ($\Delta=0$).
    }
    \label{fig:linear_combined}
\end{figure}
\vspace{-2mm}
\subsection{Linear Regression on Simulated Data}

\paragraph{Setting.}
We consider the linear model $y =  x^{\top}\theta^\star + \xi$, with $\xi \sim \mathcal{N}(0,1)$, $\theta^\star \in \mathbb{R}^p$, and $x \in \mathbb{R}^p$. 
We first fit an OLS estimator on the real dataset $(X^0, Y^0)$, and then perform multiple rounds of synthetic retraining,
where the synthetic covariate design is aligned with the right singular vectors of $X^0$.
\vspace{-3mm}
\paragraph{One-step Synthetic Retraining.}
Figures~\ref{fig:linear_one_round_theory} and \ref{fig:linear_one_round_empirical} empirically validates Theorem~\ref{thm:linear:one_step} by comparing the real-data estimator $\hat{\theta}^0$ with the one-step synthetic estimator $\hat{\theta}^1$, where color regions reflect different levels of error reduction, measured by $\log(\frac{||\hat{\theta}^0 - \theta^\star|| }{||\hat{\theta}^1 - \theta^\star||})$.
Theoretical predictions (Figure~\ref{fig:linear_one_round_theory}) align closely with empirical results (Figure~\ref{fig:linear_one_round_empirical}), validating the sharpness of our bounds.
\textcolor{black}{Using 100 real and 800 synthetic samples, we set $\theta^\star = \mathbf{1}_8$ and the verifier center $\theta_c = \theta^\star + \Delta \cdot u$, where $u$ is a random unit vector and $\Delta$ controls bias magnitude.}
The radius $r$ (see~\eqref{equ:linear:verify_condition}) determines verifier selectivity.
These results confirm that synthetic retraining outperforms the baseline under small bias (green region) but degrades under excessive bias (red region), consistent with the predicted short-term bias--variance trade-off.

\vspace{-3mm}
\paragraph{Iterative Synthetic Retraining.}
Figure~\ref{fig:linear1} validates Theorem~\ref{thm:linear:long_term}, showing that under a biased verifier ($\Delta=1$), the retrained estimator converges to the verifier's knowledge center $\theta_c$ rather than the true parameter $\theta^\star$.
In this 60-round experiment, sample size increases linearly from 100 to 5500, and results show that a more selective verifier (smaller $r$) accelerates convergence.
Figure~\ref{fig:no_bias_linear_regression} presents the unbiased case ($\Delta=0$), where verifier-based retraining consistently outperforms unfiltered baselines.
Together, these findings illustrate how the verifier's contraction effect sustains error reduction and prevents model collapse.
\textcolor{black}{Additional experiments confirm that our conclusions are robust to random covariate designs (Appendix~\ref{sec:random_synthetic_data}) and different verifier shapes (Appendix~\ref{sec:verifier_shape}).}

\subsection{Variational Autoencoders (VAEs) on MNIST}
Extending beyond linear regression, we evaluate our theory on real-world image generation tasks using Variational Autoencoders (VAEs) on the MNIST dataset.
\vspace{-3mm}
\paragraph{Setting and evaluation metrics.} 
\textcolor{black}{Specifically, we adopt Conditional VAEs (CVAEs) to leverage class conditioning and avoid verifier-induced imbalance; otherwise, easily generated digits would dominate the retained synthetic data.}
To test the bias--variance trade-off and verifier information injection, we initialize the CVAE with only 500 real images (a challenging \emph{small dataset} scenario). 
A discriminator, trained on varying amounts of real data alongside an equal number of synthetic samples, serves as the verifier.
It assigns a reality probability to each synthetic sample, from which we retain the top 10\% per digit.
Motivated by our one-step analysis, this 10\% threshold optimally balances sample quality against diversity.
The CVAE is iteratively retrained on verified data for 40 iterations until performance stabilizes.
The synthetic sample size $n_k$ follows either a fixed or linear growth schedule.
We evaluate generative performance using Fr\'echet Inception Distance (FID) \citep{heusel2017gans} and the evidence lower bound (ELBO) \citep{kingmaauto2014}.
\textcolor{black}{Appendices~\ref{app:detail} and \ref{sec:initial_sample_size} provide implementation details and initial sample size ablations.}
\vspace{-3mm}
\paragraph{Results.}

Since our verifier emphasizes perceptual realism over likelihood calibration, we report FID as the primary metric and defer qualitatively similar ELBO results to Appendix~\ref{app:detail}.
As shown in Figure~\ref{fig:fid}, synthetic retraining with a strong verifier (trained on 60K real images) yields rapid early FID improvement before plateauing, whereas unverified retraining (dashed curves) leads to severe degradation.
This validates our theory: early gains arise from the short-term bias-variance trade-off (Theorem~\ref{thm:linear:one_step}), and the verifier's contraction effect drives convergence to a fixed—though biased—point that substantially improves upon the initial CVAE (Theorem~\ref{thm:linear:long_term}).
The eventual plateau reflects two limitations: training the initial generator on only 500 real images inherently restricts overall diversity, and the standard MLP verifier lacks diversity-preserving mechanisms, introducing selection bias by disproportionately rejecting harder-to-generate modes.
For reference, a baseline CVAE trained on 60K real images achieves 17.56 FID, whereas the best synthetic model reaches 21.17 after 40 iterations.
Figure~\ref{fig:verifier} further shows that with 20K synthetic samples per round, stronger verifiers (trained on more real data) produce larger FID improvements, while weaker ones cause early plateau or even degradation. 
Qualitative results are shown in Figure~\ref{fig:mnist-illustration}, with consistent MNIST-specific FID results reported in Appendix~\ref{sec:mnist-fid}.

\begin{figure}
    \centering
    \begin{subfigure}[t]{0.495\textwidth}
        \centering
        \includegraphics[width=\textwidth]{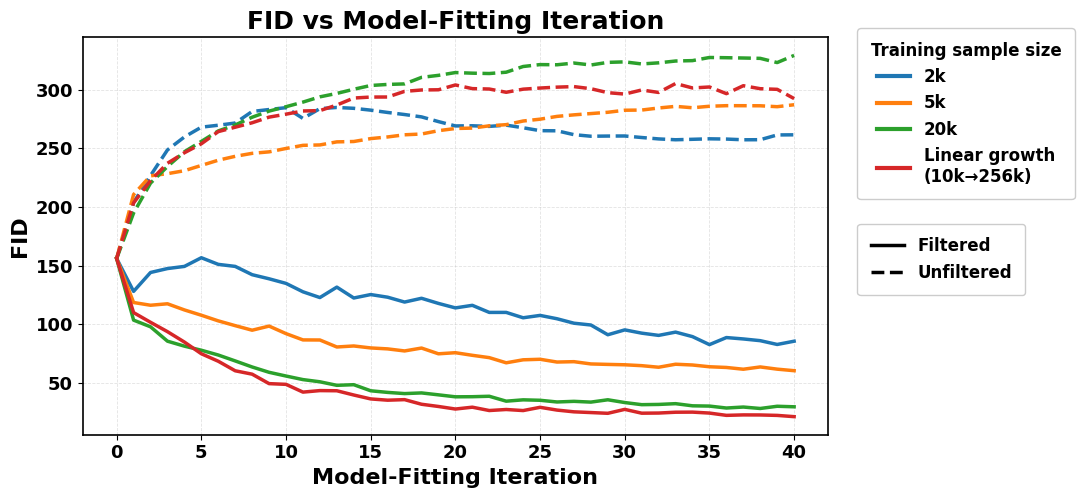}
        \caption{Filtering and training sample size.}
        \label{fig:fid}
    \end{subfigure}
    \hfill
    \begin{subfigure}[t]{0.495\textwidth}
        \centering
        \includegraphics[width=\textwidth]{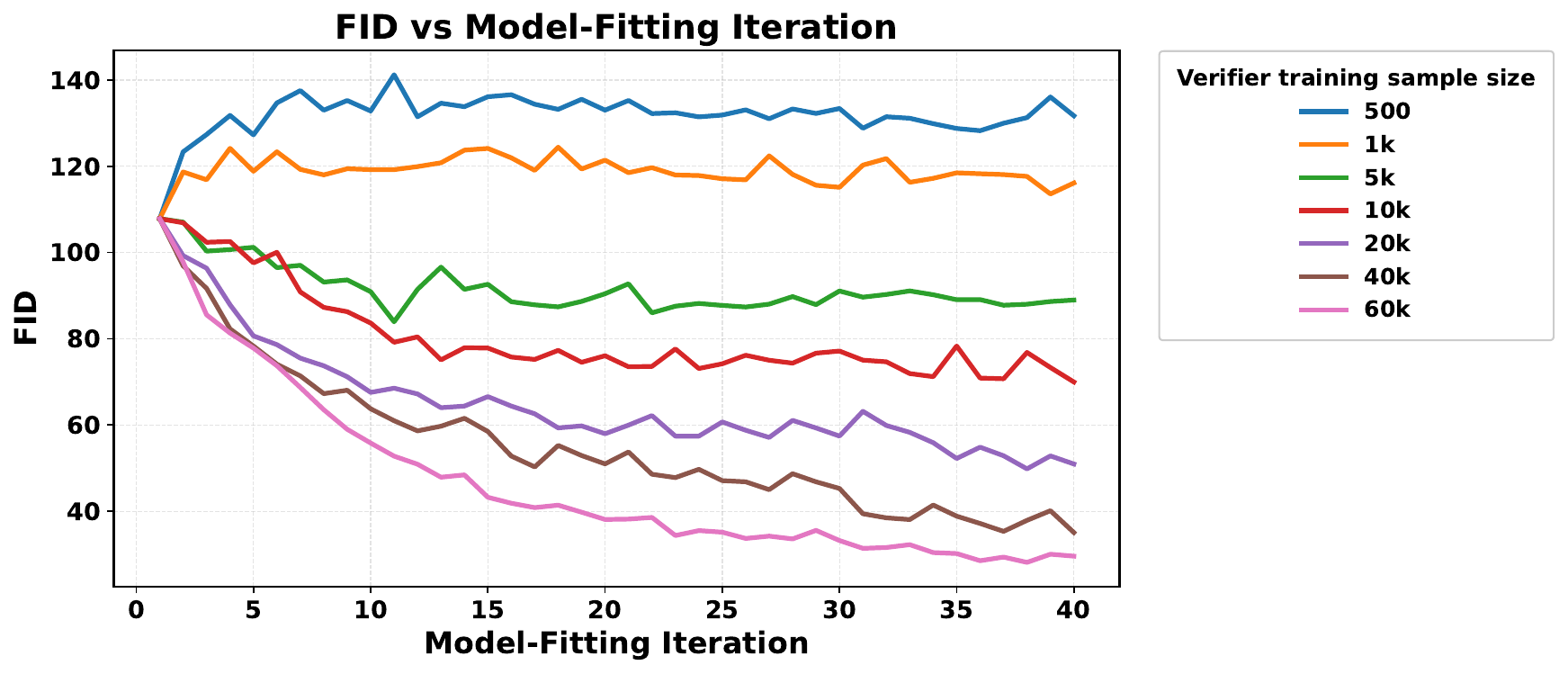}
        \caption{Verifier quality.}
        \label{fig:verifier}
    \end{subfigure}
    \caption{FID results across retraining rounds. 
    (a) Effect of filtering and retained sample size. 
    (b) Effect of verifier quality, varied by training data size. 
    Together, the plots highlight how both sample selection and verifier strength shape generative performance.}
    \label{fig:fid_verifier}
\end{figure}

\subsection{Large-Scale News Summarization} \label{sec:llm_experiment}

Extending beyond image generation, we evaluate our theory on natural-language tasks using the \textsc{XSUM} news-summarization dataset~\citep{xsum-emnlp}.

\paragraph{Setting and evaluation metrics.}

\textcolor{black}{
We use the the pretrained \texttt{SmolLM2-135M} model~\citep{allal2025smollm2smolgoesbig} as our generator, first fine-tuning it on 12.5\% of the \textsc{XSUM} training set for one epoch using full-parameter training.
While following the base setup of \cite{feng2024beyond}, we depart from their single-round evaluation by considering a \emph{multi-iteration} generate--verify--retrain regime to capture how performance evolves over repeated cycles.} Given the low-entropy nature of news summarization, we follow the common practice of employing greedy decoding for both generation and evaluation.

\begin{wrapfigure}{r}{0.45\linewidth}
    \centering
    \includegraphics[width=\linewidth]{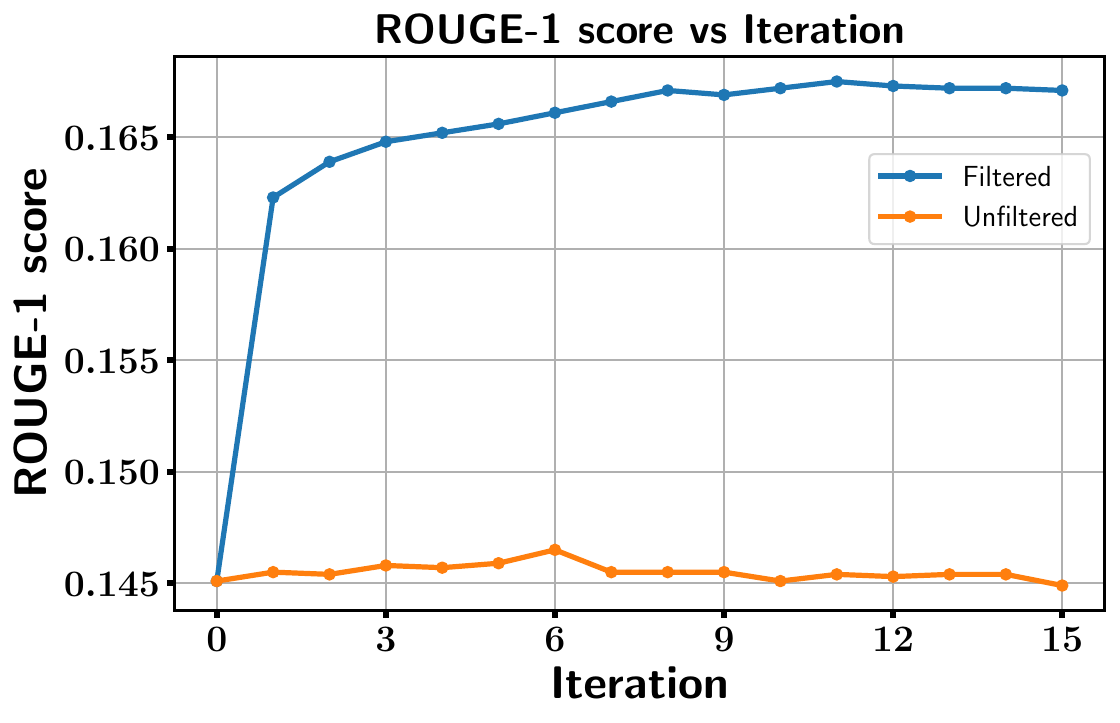}
    \caption{\textbf{ROUGE-1 score vs.\ iteration on the \textsc{XSUM} dataset.}
    Filtered retraining yields consistent early improvements, whereas unfiltered retraining shows no significant gain.
    }
    \label{fig:xsum_rouge_iterations}
\end{wrapfigure}
In each iteration, the model generates synthetic summaries for the training corpus, which are scored using ROUGE-1 against ground-truth references. 
Acting as an oracle verifier, we select the top 12.5\% of these summaries to form the synthetic retraining dataset. The model is then retrained on this subset, and test ROUGE-1 scores are recorded after each iteration.

\paragraph{Results.} Figure~\ref{fig:xsum_rouge_iterations} reports ROUGE-1 scores across 15 rounds for both filtered and unfiltered retraining regimes, while all other experimental conditions remain identical. 
Under verifier filtering, performance improves monotonically during early iterations before eventually stabilizing. 
In contrast, the unfiltered baseline fluctuates around its initial score without meaningful improvement. 
These observed dynamics mirror our MNIST findings and theoretical predictions, demonstrating that our framework scales to natural-language tasks.

\section{Discussion}

Our study provides a theoretical and empirical characterization of verifier-guided synthetic retraining.  
We show that the process yields \emph{short-term gains} by reducing variance through verifier filtering, but in the \emph{long run} the estimator converges to the verifier’s knowledge center.  
This explains both the promise and the risk of such methods: a high-quality verifier can inject reliable external knowledge, while a biased verifier inevitably steers the model away from the truth.  
Viewed through the lens of \emph{information elicitation}, our framework formalizes how external signals are incorporated recursively into training and why the outcome reflects the verifier’s information.  


Meanwhile, we also acknowledge the limitations of our results.Primarily, our analytical testbed relies on a well-specified parametric setting (linear regression) that assumes the existence of a global, ground-truth optimal parameter $\theta^*$.
This idealized assumption mathematically abstracts the complex, often competing attributes of a "good" generative model---such as sample diversity and generation quality---into a singular distance metric.
In practice, while models like LLMs are parametric, they are approximating a true data-generating process that is unknown and highly likely non-parametric.
For complex domains like natural language, a singular "true model" may not even exist; therefore, defining or evaluating an optimal model remains an open challenge.
While our empirical extensions to VAEs and LLMs validate the theory qualitatively, formal generalization to richer models such as exponential families or simple neural network architectures are interesting future directions.
Other future venues include developing sharper bounds for nonlinear models, exploring effectiveness of alternative synthetic design strategies beyond block orthogonalization, and studying verifier dynamics in large language models (LLMs) and vision models.   

\clearpage
\bibliographystyle{iclr2026ready_conference}
\bibliography{reference}

\clearpage
\appendix
\section*{Appendix}
This appendix contains: Appendix~\ref{app:gaussian} (1-D Gaussian toolkit),  Appendix~\ref{app:linear} (reduction and full proof for linear regression), Appendix~\ref{app:detail} (additional details on VAE experiments), Appendix~\ref{app:addexp} (additional simulations and experiments)

\section{One-Dimensional Gaussian Toolkit}\label{app:gaussian}
In this section, we provide a toolkit for analyzing the one-dimensional Gaussian mean estimation problem with verifier-filtered synthetic data.
This toolkit serves as the foundation for our analysis of the linear regression models.
We will establish several key lemmas and theorems that characterize the MSE of the mean estimator under the one-dimensional Gaussian model.
These results will be instrumental in proving Theorem~\ref{thm:linear:one_step} and Theorem~\ref{thm:linear:long_term} in Appendix~\ref{app:linear}.
\subsection{Setup and Notations}
We consider the one-dimensional mean estimation problem where the real data \( X^0_1, \dots, X^0_{n_0} \) are independently and identically distributed (i.i.d.) from a Gaussian distribution:
\[
X^0_1, \dots, X^0_{n_0} \overset{\text{i.i.d.}}{\sim} \mathcal{N}(\mu, \sigma^2),
\]
with known variance \( \sigma^2 \).

In our setting, a verifier exists and encodes external knowledge that the true mean lies in an interval $[a,b]$ (i.e.\( \mu \in [a, b]\)). 
Therefore, $\bar{X}^{0} = \frac{X_1 + \cdots + X_{n_0}}{n_0}$ is the empirical mean of real data, which minimizes MSE if \emph{no extra} information is supplied.
We are interested in whether data verification could effectively inject new information and improve over $\bar{X}^{0}$.
Consider the following synthetic data generation and filtering procedure:
\begin{itemize}
     \item Generate $n_1$ synthetic data \(X^1_1, \dots, X^1_{n_1} \overset{\text{i.i.d.}}{\sim} \mathcal{N}(\bar{X}^{0}, \sigma^2)\).
     \item Retain \(X^0_i \in [a, b]\) as \(X'^1_1, \dots, X'^1_{n_1'}\), and estimate \(\mu\) using $\bar{X}^{1} = \frac{1}{n_1'} \sum_{i=1}^{n_1'} X'^1_i.$
\end{itemize}

We will compare the estimator $\bar{X}^{1}$ with $\bar{X}^{0}$ and formally characterize when data verification enhances or degrades model performance - i.e., when $\mathbb{E}(\bar{X}^{1}-\mu)^2 < \mathbb{E}(\bar{X}^{0}-\mu)^2$ or not.
Our key finding is that $\bar{X}^{1}$ introduces the core bias-variance trade-off that underpins model improvement or degradation.
We will characterize the MSE of $\bar{X}^1$ which reveals how key quantities such as the real and synthetic sample size, the verifier's bias and variance will decide performance of the filtering strategy.
These insights provide intuition for extending verifier-guided re-training to more complex settings.

We first review some notation and key results for the truncated normal distribution, which will be used in the subsequent sections.
Consider a one-dimensional normal distribution $X \sim \mathcal{N}(\mu, \sigma^2)$ and let $X'$ be its truncated version restricted to the interval $[a, b]$.
The distribution of $X'$ is the called the \textsl{truncated normal distribution}, denoted as $X' \sim \mathcal{N}(x|\mu, \sigma^2) \cdot \mathbbm{1}_{\{a<x<b\}}$.
The mean and variance of the truncated normal distribution $X'$ are given analytically:
\begin{align}
    \E[X'|\mu] &= \mu - \sigma \frac{\phi(\frac{b-\mu}{\sigma}) - \phi(\frac{a-\mu}{\sigma})}{\Phi(\frac{b-\mu}{\sigma}) - \Phi(\frac{a-\mu}{\sigma})} := \mu + \sigma m_1(\frac{a-\mu}{\sigma}, \frac{b-\mu}{\sigma}) \notag\\
    \Var(X'|\mu) &= \sigma^2 \left[ 1 - \frac{\frac{b-\mu}{\sigma} \phi(\frac{b-\mu}{\sigma}) - \frac{a-\mu}{\sigma} \phi(\frac{a-\mu}{\sigma})}{\Phi(\frac{b-\mu}{\sigma}) - \Phi(\frac{a-\mu}{\sigma})} - \left( \frac{\phi(\frac{b-\mu}{\sigma}) - \phi(\frac{a-\mu}{\sigma})}{\Phi(\frac{b-\mu}{\sigma}) - \Phi(\frac{a-\mu}{\sigma})} \right)^2 \right]\notag\\
    &:= \sigma^2 m_2(\frac{a-\mu}{\sigma},\frac{b-\mu}{\sigma}) \label{equ:trunc_normal}
\end{align}
where $\phi(x)$ and $\Phi(x)$ denote the standard normal density and cumulative distribution functions, respectively.
Standardizing $X$ via $Z := \frac{X-\mu}{\sigma}$ and setting
\begin{align}
\alpha = \frac{a - \mu}{\sigma}, \quad \beta = \frac{b - \mu}{\sigma}, \label{equ:alpha_beta}
\end{align}
the expression in \eqref{equ:trunc_normal} become:
\begin{align}
    \mathbb{E}[Z'] &= m_1(\alpha, \beta) \notag\\
    \Var(Z') &= m_2(\alpha, \beta) \label{equ:trunc_normal_standard}
\end{align}
where $Z'\sim \mathcal{N}(x|0, 1) \cdot \mathbbm{1}_{\{\alpha<x<\beta\}}$ is the standardized truncated normal distribution.
For convenience, we write $\mathcal{N}_{trunc}(\alpha, \beta) := \mathcal{N}(x|0, 1) \cdot \mathbbm{1}_{\{\alpha<x<\beta\}}$.
Thus, $m_1$ and $m_2$ correspond to the first and second central moments of the standardized truncated normal distribution.
In addition, we also define the third central moment  of the standardized truncated normal distribution:
\begin{align}\label{equ:trunc_normal_standard2}
    m_3(\alpha, \beta) &:= \E(Z'-\E Z')^3 \notag\\
&= -\frac{(\beta^2 - 1)\phi(\beta) - (\alpha^2 - 1)\phi(\alpha)}{(\Phi(\beta) - \Phi(\alpha))} - \frac{3(\phi(\beta) - \phi(\alpha))(\beta\phi(\beta) - \alpha\phi(\alpha))}{(\Phi(\beta) - \Phi(\alpha))^2} \notag\\
&\quad - \frac{2(\phi(\beta) - \phi(\alpha))^3}{(\Phi(\beta) - \Phi(\alpha))^3}.
\end{align}
In particular, $0<m_2(\alpha,\beta)<1$ for any $\alpha<\beta$ and $m_1(\alpha,\beta)=m_3(\alpha,\beta)=0$ if $\alpha+\beta=0$.

\subsection{Characterization of $\E(\bar{X}^{1}-\mu)^2$, Bias-Variance Trade-off, and Model Improvement}

\begin{theorem}\label{thm:gaussian:main}
    Assume that \( n_1 > n_0 \geq 100\). Then there exists constant \( K\), depending only on \(\alpha\) and \(\beta\), such that
    \begin{align}
         &\left| \frac{1}{\sigma^2}\mathbb{E}(\bar{X}^1-\mu)^2 - \underbrace{\frac{m_2(\alpha,\beta)}{n_1}}_{\textcolor{red}{\text{Synthetic Variance}}} - \underbrace{\bigg(m_1^2(\alpha,\beta) + \frac{m^2_2(\alpha,\beta)+m_3(\alpha,\beta)m_1(\alpha,\beta)}{n_0}\bigg)}_{\textcolor{red}{\text{Verification Bias+Variance}}} \right| \notag\\
         & < K\left(\frac{1}{n_1n^{1/3}_0}+\frac{1}{n^{3/2}_0}\right)\label{equ:high_prob_y'}
    \end{align}
     holds with probability at least \( 1 - \exp\left(-\frac{1}{2}n_0^{1/3}\right) \).
\end{theorem}
\begin{proof}[Proof of Theorem~\ref{thm:gaussian:main}]

It will be convenient to reparameterize the sample mean estimators by centering them around the true mean. Specifically, we define the residuals:
\begin{align}
    \epsilon_1 := \frac{\bar{X}^{0}-\mu}{\sigma}, \quad \epsilon_1 \sim \mathcal{N}(0, \frac{1}{n_0}). \label{equ:epsilon}
\end{align}
Note that \(\bar{X}^{1}\) is the mean of \(n_1\) i.i.d. samples from the truncated normal distribution \(\mathcal{N}(x|\bar{X}^0,\sigma^2)\cdot \mathbbm{1}_{\{a<x<b\}}\). 
The MSE of \( \bar{X}^1 \) can be decomposed as follows:
\begin{align}
\mathbb{E}[(\bar{X}^1 - \mu)^2] 
&= \mathbb{E}_{\bar{X}^{0}} \, \mathbb{E}_{\bar{X}^1 \mid\bar{X}^{0}} \left[(\bar{X}^1 - \mu)^2\right] \notag \\
&= \mathbb{E}_{\bar{X}^{0}} \left[ \operatorname{Var}(\bar{X}^1 \mid, \bar{X}^{0}) + \left( \mathbb{E}[\bar{X}^1 \mid \bar{X}^{0}] - \mu \right)^2 \right] \notag \\
&= \sigma^2 \mathbb{E}_{\bar{X}^{0}} \left[   \frac{m_2(\alpha-\epsilon_1, \beta-\epsilon_1)}{n_1}  \right] + \mathbb{E}_{\bar{X}^{0}} \left[ (\bar{X}^{0} - \mu - \sigma m_1(\alpha-\epsilon_1, \beta-\epsilon_1))^2 \right] \notag \\
&= \frac{\sigma^2}{n_1} \mathbb{E}_{\epsilon_1}\left[m_2(\alpha-\epsilon_1, \beta-\epsilon_1)\right] + \sigma^2 \, \mathbb{E}_{\epsilon_1} \left[ \left(m_1(\alpha - \epsilon_1, \beta - \epsilon_1) + \epsilon_1 \right)^2 \right] \label{equ:yrisk2}
\end{align}
       
For the first term in \ref{equ:yrisk2}, we consider the event \( E_1 := \left\{ |\epsilon_1| < n_0^{-1/3} \right\} \), the function \( m_2(\cdot, \cdot) \) is Lipschitz continuous in a neighborhood of \( (\alpha, \beta) \), so we have
\begin{equation}\label{equ:first}
    \left| m_2(\alpha - \epsilon_1, \beta - \epsilon_1) - m_2(\alpha, \beta) \right| 
    = |\epsilon_1| \cdot \left| m_2^{(1)}(\alpha - \xi, \beta - \xi) \right| 
    < \frac{M_1}{n_0^{1/3}},
\end{equation}
  
    for some \( \xi \in (0, \epsilon_1) \), where we define 
    \[
    M_1 := \sup_{|\xi| < \frac{1}{100^{\frac{1}{3}}}} \left| m_2^{(1)}(\alpha - \xi, \beta - \xi) \right|,
    \]
    and \( M_1 \) is a constant independent of \( n_0 \) as long as \( n_0 \geq 100\).
Event $E_1$ hold with high probability:
\[
\mathbb{P}\left(|\epsilon_1| < n_0^{-1/3}\right) >1- \frac{\exp\left(-\frac{n_0^{1/3}}{2}\right)}{\sqrt{\pi / 2} \cdot n_0^{1/6}} >1- \frac{\exp\left(-\frac{n_0^{1/3}}{2}\right)}{\sqrt{\pi / 2} \cdot 100^{1/6}} >  1-\exp\left(-\frac{n_0^{1/3}}{2}\right).
\]
Then we consider then second term in \ref{equ:yrisk2}. The Taylor expansion of the function
\[
m_1(\epsilon_1) := m_1(\alpha - \epsilon_1, \beta - \epsilon_1)
\]
up to the third-order terms is:
\begin{align}\label{equ:taylor}
m_1(\epsilon_1) 
&= m_1(\alpha,\beta) - \left[ 1 - m_2(\alpha, \beta) \right] \epsilon_1 +\frac{1}{2}m_3(\alpha,\beta) \epsilon_1^2 + \frac{1}{6} m_1^{(3)}(\xi) \epsilon_1^3, \quad \text{for some } \xi \in (0, \epsilon_1),
\end{align}
where \( m_1^{(3)}(\xi) \) denotes the third derivative of \(m_1\) evaluated at some point between \(0\) and \(\epsilon_1\).
Then we can get 
\begin{align}
    \mathbb{E}_{\epsilon_1} \left[ \left(m_1(\alpha - \epsilon_1, \beta - \epsilon_1) + \epsilon_1 \right)^2 \right] &= \mathbb{E}\left( m_1(\alpha,\beta) + m_2(\alpha, \beta) \epsilon_1 + \frac{1}{2}m_3(\alpha,\beta) \epsilon_1^2 + \frac{1}{6} m_1^{(3)}(\xi) \epsilon_1^3 \right)^2\notag\\
&= m_1^2(\alpha,\beta) +\frac{m_2^2(\alpha,\beta) + m_1(\alpha,\beta)m_3(\alpha,\beta)}{n_0} + \frac{3m_3^2(\alpha,\beta)}{4n^2_0} \notag \\
\qquad & + \mathbb{E}\left( m_1(\alpha,\beta) +m_2(\alpha,\beta)\epsilon_1 + \frac{1}{2}m_3(\alpha,\beta)\epsilon^2_1\right) \frac{m_1^{(3)}(\xi)}{3}\epsilon_1^3 \notag \\
\qquad & + \mathbb{E}\left( \frac{m_1^{(3)^2}(\xi)}{36}\epsilon_1^6\right). \label{equ:infinite_mse}
\end{align}

First, using the fact that there exists constant $M$ that only depends on $\alpha$ and $\beta$, such that $|m_1^{(3)}(x)| < M$ for any $x$, we have:
\begin{align}
&\left|\mathbb{E}\left[\left( m_1(\alpha,\beta) + m_2(\alpha,\beta)\epsilon_1 + \frac{1}{2}m_3(\alpha,\beta)\epsilon_1^2 \right) \frac{m_1^{(3)}(\xi)}{3} \epsilon_1^3 \right] \right| \notag \\
\leq\;& \mathbb{E}\left[\left( |m_1(\alpha,\beta)| + m_2(\alpha,\beta)|\epsilon_1| + \frac{1}{2}|m_3(\alpha,\beta)|\epsilon_1^2 \right) \cdot \frac{M}{3} |\epsilon_1|^3 \right] \notag \\
=\;& \mathbb{E}\left[ \frac{M}{3} |m_1(\alpha,\beta)| |\epsilon_1|^3 + \frac{M}{3} m_2(\alpha,\beta) |\epsilon_1|^4 + \frac{M}{6} |m_3(\alpha,\beta)| |\epsilon_1|^5 \right] \notag \\
\leq\;& \frac{K_1}{n_0^{3/2}}. \notag
\end{align}
for some constant \( K_1 \) depending only on \( \alpha \) and \( \beta \).

Secondly, the last term in \eqref{equ:infinite_mse} is bounded by:
\[
\mathbb{E} \left[ \frac{{m_1^{(3)}}^2(\xi)}{36} \epsilon_1^6 \right]
\leq \frac{M^2}{36} \mathbb{E}[\epsilon_1^6] 
= \frac{5 M^2 \sigma^6}{12 n_0^3} \leq \frac{K_2}{n_0^3},
\]
for some constant \( K_2 \).

So the second term in \ref{equ:yrisk2} is bounded by
\begin{equation}\label{equ:second}
    \left| \mathbb{E}_{\epsilon_1} \left[ \left(m_1(\alpha - \epsilon_1, \beta - \epsilon_1) + \epsilon_1 \right)^2 \right] 
- m_1^2(\alpha, \beta) 
- \frac{m_2^2(\alpha, \beta) + m_1(\alpha,\beta)m_3(\alpha, \beta)}{n_0}
\right| 
< \frac{K}{n_0^{3/2}}
\end{equation}
for some constant \( K \).

Combining \ref{equ:yrisk2}, \ref{equ:first}, and \ref{equ:second} completes the proof.

\end{proof}

\subsection{Iterative Retraining and Long-Term Dynamics in One-Dimensional Gaussian Mean Estimation}

Now consider the verifier-guided synthetic retraining in the Gaussian mean estimation setting.
The iterative retraining process can be described by the following algorithm.

\begin{algorithm}[H]
\caption{Iterative Verifier-Guided Retraining for Gaussian Mean Estimation}
\label{alg:gaussian_mean}
\begin{algorithmic}[1]
  \State \textbf{Input:} Initial estimate $\bar X^{0}$ from real data
  \For{$k = 0,1,2,\dots$}
    \State Draw $\xi_i \stackrel{\text{i.i.d.}}{\sim}\mathcal N(0,\sigma^2)$ and construct synthetic samples $X_i^{k} = \bar X^{k} + \xi_i$.
    \State Retain points with $a < X_i^{k} < b$, yielding $n_k$ verified samples $\{X'^k_i:i=1,2,\ldots n_k\}$.
    \State $\bar X^{k+1} \gets \frac{1}{n_k}\sum_{i=1}^{n_k} X'^k_i$.
  \EndFor
\end{algorithmic}
\end{algorithm}

Algorithm \ref{alg:gaussian_mean} defines a Markov process $\{\bar{X}^0,\bar{X}^1, \dots \bar{X}^k, \dots\}$, where the conditional distribution $p(\bar{X}^{k+1} |\bar{X}^{k})$ is given by
\begin{align}\label{equ:gaussian:transition}
p(\bar{X}^{k+1} |\bar{X}^{k}): \bar{X}^{k+1} = \bar{X}^k + \sigma \frac{\sum_{i=1}^{n_k}\xi'^{k+1}_i}{n_k}, \qquad \xi'^{k+1}_i \;\text{i.i.d}\; \sim \mathcal{N}_{trunc}(\frac{a - \bar{X}^k}{\sigma}, \frac{b - \bar{X}^k}{\sigma})
\end{align}

The following theorem summarizes these findings:

\begin{theorem}
\label{thm:gaussian:limit}
Let $\bar{X}^k$ be the Markov process determined by \eqref{equ:gaussian:transition} with initial condition
\begin{align}
    \bar{X}^0 \sim \mathcal{N}(0,\frac{\sigma^2}{n_0}),\notag
\end{align}
and assume $n_k$ is non-decreasing in $k$.
Then the following statements hold:

\begin{itemize} 
    \item If $|a|,|b| < \infty$, there exists a constant $0<\rho < 1$ such that, 
    \[\mathbb{E}\left(\bar{X}^k - \frac{a+b}{2}\right)^2 \leq \rho^{2k}\mathbb{E}(\bar{X}^0-\frac{a+b}{2})^2+\sum_{j=0}^{k-1}\frac{\rho^{2(k-j)-1}}{n_j}.\]
    Moreover, if $\lim_{k \to \infty} n_k = \infty$, $\lim_{k \to \infty} \mathbb{E}|\bar{X}^k - \frac{a+b}{2}|^2 = 0$.
    \item If $-\infty=a<b<\infty$, then $\liminf_{k \to \infty} \bar{X}^k = -\infty$. If $-\infty<a<b=\infty$, then $\limsup_{k \to \infty} \bar{X}^k = \infty$.
\end{itemize}

\end{theorem}
\begin{proof}[Proof of Theorem~\ref{thm:gaussian:limit}]
Define
\begin{align}
    \epsilon_k = \frac{\bar{X}^k - \mu}{\sigma}\label{equ:def_epsilon},
\end{align}
which represents the standardized error of the estimator $\bar{X}^k$.
It is easy to see that $\epsilon_k \in [\alpha,\beta] \Leftrightarrow \bar{X}^k \in [a,b]$, where $\alpha,\beta$ are defined in \eqref{equ:alpha_beta}.
Therefore, it suffices to consider the standardized process $\{\epsilon_k,k=0,1,2,\ldots\}$.
\eqref{equ:gaussian:transition} can be standardized as:

\begin{align}
     \epsilon_{k+1}  = \epsilon_{k} + \frac{\sum_{i=1}^{n_k}\xi'^{k+1}_i}{n_k}, \qquad \xi'^{k+1}_i \sim \mathcal{N}_{\text{trunc}}\left(\alpha - \epsilon_k,\beta - \epsilon_k \right)\label{equ:gaussian:transition_standard},
\end{align}

For convenience, we shift the noise terms $\xi'^{k+1}_i$ in \eqref{equ:gaussian:transition_standard} to have mean zero. Therefore, we introduce

\begin{align}\label{equ:gaussian:Tv_def}
T_{\alpha,\beta}(x):=x+\mathbb{E}\!\big[Z\;\big|\;\alpha-x\le Z\le \beta-x\big],\qquad
v_{\alpha,\beta}(x):= \mathrm{Var}\!\big(Z\;\big|\;\alpha-x\le Z \le\beta-x\big).
\end{align}
where $Z\sim\mathcal{N}(0,1)$.

Therefore, \eqref{equ:gaussian:transition_standard} can be rewritten as
\begin{equation}\label{equ:gaussian:contraction}
\epsilon_{k+1}=T_{\alpha,\beta}(\epsilon_k)+\eta_{k+1}
\end{equation}
where $\eta_{k+1} = \frac{1}{n_k}\sum_{i=1}^{n_k}\left(\xi'^{k+1}_i - \mathbb{E}\xi'^{k+1}_i\right)$ is the average of independent mean zero noise in \eqref{equ:gaussian:transition_standard}. In particular, we have
\[
\mathbb{E}[\eta_{k+1} \mid \mathcal{F}_k] = 0,\qquad
\mathrm{Var}(\eta_{k+1} \mid \mathcal{F}_k) = \frac{v_{\alpha,\beta}(\epsilon_k)}{n_k}.
\]

where $\mathcal{F}_k := \sigma(\epsilon_0, \eta_1, \dots, \eta_k)$ and $n_k$ is the (post-filtering) batch size at round $k$.

It is easy to see that 
\begin{align}
    T_{\alpha,\beta}(x) &= x + m_1(\alpha - x,\beta - x),\notag\\
    v_{\alpha,\beta}(x) &= m_2(\alpha - x,\beta - x),\notag\\
    T_{\alpha,\beta}'(x) &= v_{\alpha,\beta}(x). \notag
\end{align}

We first consider $|a|,|b| < \infty$. In this case, we first show that the derterministic part $T_{\alpha,\beta}(x)$ in \eqref{equ:gaussian:contraction} is a global contraction.
Since $-\infty < \alpha < \beta <\infty$, we have

\[
\sup_{x \in \mathbb{R}} T_{\alpha,\beta}'(x) = \sup_{x \in \mathbb{R}} \mathrm{Var}\big(Z \mid\, \alpha - x \le Z \le \beta-x\big)  = \mathrm{Var}\big(Z \mid\, |Z| < |\frac{\alpha+\beta}{2}|\big):= \rho < 1.
\]

Therefore, $T_{\alpha,\beta}(x)$ is a global contraction. By the contractive mapping theorem that $T_{\alpha,\beta}(x)$ has a unique fixed point $x^*$, which solves $x^* = T_{\alpha,\beta}(x^*)$.
It is easy to see that
\begin{align}
x^* = T_{\alpha,\beta}(x^*) \implies x^* = x^* + \mathbb{E}(Z \big| \alpha - x^*\le Z \le \beta - x^*) \implies x^* = \frac{\alpha+\beta}{2}.
\end{align}

By the mean-value theorem,
\begin{equation}
|T_{\alpha,\beta}(\epsilon_k)-\frac{\alpha+\beta}{2}|\,\le\,\rho\,|\epsilon_k-\frac{\alpha+\beta}{2}|.\notag
\end{equation}

Let $V_k:=(\epsilon_k-\frac{\alpha+\beta}{2})^2$, we have
\begin{align}
\mathbb{E}\!\big[V_{k+1}\mid \epsilon_k\big]
=(T_{\alpha,\beta}(\epsilon_k)-\frac{\alpha+\beta}{2})^2+\frac{v_{\alpha,\beta}(\epsilon_k)}{n_k}
\;\le\;\rho^2 (\epsilon_k-\frac{\alpha+\beta}{2})^2+\frac{\rho}{n_k}.\notag
\end{align}
Taking expectations yields
\begin{equation}\label{equ:gaussian:onetep}
\mathbb{E} V_{k+1}\;\le\;\rho^2\,\mathbb{E} V_k\;+\;\frac{\rho}{n_k}\,.
\end{equation}

Unrolling \eqref{equ:gaussian:onetep},
\begin{align}\label{equ:gaussian:bound}
\mathbb{E} V_k
\;\le\;\rho^{2k}\mathbb{E}V_0\;+\;\rho\sum_{j=0}^{k-1}\frac{\rho^{2(k-1-j)}}{n_j}.
\end{align}
It is easy to see that 
\begin{align}
\mathbb{E} V_k
\;\le\; \rho^{2k}\mathbb{E}V_0\;+\;\rho\sum_{j=0}^{k-1}\frac{\rho^{2(k-1-j)}}{n_0} < \rho^{2k}\mathbb{E}V_0\;+\;\frac{\rho}{n_0(1-\rho^2)}.\notag
\end{align}

Therefore, by the Cauchy-Schwarz inequality, $\lim_{k \to \infty} \mathbb{E} \epsilon^2_k < \infty$ easily follows.
Moreover, when $n_k \to \infty$, let $g_i:=\rho^{2i}$ and $a_j:=1/n_j\to 0$. A standard $\ell^1$-convolution argument shows $(g*a)_k:=\sum_{j=0}^{k-1} g_{k-1-j}a_j = \sum_{j=0}^{k-1}\frac{\rho^{2(k-1-j)}}{n_j} \to 0$.
Therefore $\lim_{k \to \infty}\mathbb{E} V_k = \lim_{k \to \infty}\mathbb{E} (\epsilon_k - \frac{\alpha+\beta}{2})^2 = 0$.

Now we consider the case $-\infty= a < b < \infty$ (equivalently $-\infty= \alpha < \beta < \infty$). We will show that $\liminf_{k \to \infty} \epsilon_k = -\infty \text{ a.s.}$.

Let $t_k := \beta - \epsilon_k$ and the recursion \eqref{equ:gaussian:contraction} can be rewritten for $t_k$:
\[
t_{k+1} = t_k + \lambda(t_k) - \eta_{k+1},
\]
where $\lambda(t_k) = -\mathbb{E}(Z|Z < \beta - \epsilon_k) =\mathbb{E}[Z \mid Z \geq -t_k]$.

Consider the hitting time $\tau_M := \inf\{k : t_k \geq M\}$ for any $M > 0$.
Fix $M > 0$ and define
\[
m(M) := \min_{t \leq M} \lambda(t) = \mathbb{E}[Z \mid Z \geq -M] > 0,
\]
which is strictly positive the fact that $\lambda(t) > 0$ and $\lambda(t)$ is a decreasing function.
On the event $\{\tau_M > K\}$ we have $t_j < M$ for $j = 0, \dots, K-1$, hence $\lambda(t_j) \geq m(M)$.
Summing the recursion yields
\[
t_K = t_0 + \sum_{j=0}^{K-1} \lambda(t_j) - \sum_{j=0}^{K-1} \eta_{j+1}
\geq t_0 + K\,m(M) - S_K,
\]
where $S_K := \sum_{j=0}^{K-1} \eta_{j+1}$ and $t_0 = \beta - \epsilon_0$ is $\mathcal{F}_0$-measurable (hence random). Therefore,
\begin{equation}\label{equ:gaussian:even_inclusion}
\{\tau_M > K\} \subseteq \Big\{ S_K \geq t_0 + K\,m(M) - M \Big\}.
\end{equation}
Define the (random) burn-in index
\[
K_0 := \Big\lceil \frac{2(M-t_0)}{m(M)} \Big\rceil.
\]
Then for all $K \geq K_0$,
\[
t_0 + K\,m(M) - M \geq \frac{m(M)}{2} K,
\]
and \eqref{equ:gaussian:even_inclusion} gives, conditionally on $\mathcal{F}_0$,
\begin{equation}\label{equ:gaussian:even_inclusion2}
\{\tau_M > K\} \subseteq \Big\{ S_K \geq \frac{m(M)}{2} K \Big\}, \qquad \text{for all } K \geq K_0.
\end{equation}

Next, we will show that $S_K$ is a sub-exponential random variable in event $\{\tau_M > K\}$.
Since $S_K = \sum_{j=0}^{K-1} \eta_{j+1} = \sum_{j=0}^{K-1} \frac{1}{n_j}\sum_{i=1}^{n_j}\left(\xi'^{j+1}_i -\mathbb{E}\xi'^{j+1}_i\right)$, we will first show that $\xi'^{j+1}_i -\mathbb{E}\xi'^{j+1}_i$ is sub-exponential.

Since $\xi'^{j+1}_i \sim \mathcal{N}_{\text{trunc}}(-\infty,\beta - \epsilon_j) = \mathcal{N}_{\text{trunc}}(-\infty, t_j)$, 
on the event $\{\tau_M > K\}$ we have 
\begin{align}
\xi'^{j+1}_i - \mathbb{E}\xi'^{j+1}_i < t_j - \mathbb{E}[Z\mid Z < t_j] \leq M - \mathbb{E}[Z\mid Z < M] := b(M) <\infty.\notag
\end{align}
The above inequality follows from the fact that $t - \mathbb{E}[Z\mid Z < t]$ is an increasing function of $t$ and 
$t_j < M$ for $j = 0, \dots, K-1$ on the event $\{\tau_M > K\}$. In addition, $\mathrm{Var}(\xi'^{j+1}_i) = \mathrm{Var}(Z|Z < t_j) \leq 1$.
Therefore, $\xi'^{j+1}_i - \mathbb{E}\xi'^{j+1}_i$ is mean zero, bounded above by $b(M)$ with $\mathrm{Var}\left(\xi'^{j+1}_i - \mathbb{E}\xi'^{j+1}_i\right) < 1$.
By Bennet/Bernstein MGF inequality, we have
\begin{align}
    \log \mathbb{E}e^{\lambda(\xi'^{j+1}_i - \mathbb{E}\xi'^{j+1}_i)} \leq \frac{\lambda^2}{2(1 - b(M)\lambda/3)},\notag
\end{align}
for $0<\lambda<\frac{3}{b(M)}$. This shows that $\xi'^{j+1}_i - \mathbb{E}\xi'^{j+1}_i$ is sub-exponential with parameters $SE(1,2b(M)/3)$.
By standard properties of sub-exponential random variables, $\eta_{j+1} = \frac{1}{n_j}\sum_{i=1}^{n_j}\left(\xi'^{j+1}_i -\mathbb{E}\xi'^{j+1}_i\right)$ is $SE(1/n_j, 2b(M)/(3n_j))$ and $S_K = \sum_{j=0}^{K-1} \eta_{j+1}$ is $SE(\sum_{j=0}^{K-1} 1/n_j, 2b(M)/(3n_1))$ since $n_j$ is non-decreasing.
Therefore, for any $t > 0$ we have tail bound
\begin{align}\label{equ:gaussian:tail_bound}
\mathbb{P}\left( S_K \geq t \right) \leq \exp\left(-\frac{1}{2}\min\{\frac{t^2}{\sum_{j=0}^{K-1} 1/n_j},\frac{n_1t}{2b(M)}\}\right) \leq \exp\left(-\frac{1}{2}\min\{\frac{n_1t^2}{K},\frac{n_1t}{2b(M)}\}\right).
\end{align}
Use the tail bound \eqref{equ:gaussian:tail_bound} in \eqref{equ:gaussian:even_inclusion2}, we have 
\begin{align}
\mathbb{P}\left(\tau_M > K\mid \mathcal{F}_0\right) \leq \mathbb{P}\left( S_K \geq \frac{m(M)}{2} K \right) \leq \exp\Big(-c(M)n_1 K\Big)
\end{align}

for all $K \geq K_0$ with $c(M) = \min\left\{\frac{m(M)^2}{8},\frac{n(M)}{8b(M)}\right\}$.
\begin{align}\label{equ:gaussian:tail}
\mathbb{P}\left( \tau_M > K \right) &= \mathbb{E} \left[\mathbb{P}\left( \tau_M > K \mid \mathcal{F}_0 \right)\right]\notag\\
&\le\mathbb{E} \left[\exp\Big(-c(M)n_1K\Big)\mathbbm{1}_{\{K > K_0\}}\right]+\mathbb{P}\left( K \leq K_0 \right)
\end{align}

Let $K \to \infty$ in \eqref{equ:gaussian:tail}, we get $\mathbb{P}(\tau_M < \infty) = 1$. Since $M$ is arbitrary, this implies $\liminf_{k \to \infty} \epsilon_k = -\infty \text{ a.s.}$.

The case $-\infty < a < b = \infty$ can be proved in the same way, therefore is omitted.

\end{proof}

\clearpage

\section{Proofs of All Theorems in Section~\ref{sec:linear}}\label{app:linear}
\footnotetext{Since only the most recent round of synthetic data is retained for training, this scheme is sometimes referred to as a \emph{discard workflow} in the literature (e.g.,~\citep{dey2024universality}).}
\begin{algorithm}[H]
\captionsetup{labelfont=bf, name=Scheme} 
\caption{Iterative  Retraining with Verified Synthetic Data \protect \footnotemark}
\label{scheme:linear_retraining}
\begin{algorithmic}[1]
\State \textbf{Input:} Initial estimator $\hat{\theta}^0$ from $n_0$ real samples
\For{$k = 0,1,2,\dots$}
    \State \textbf{Generate:} Synthetic covariates $X^{k+1}$ are constructed (details below), and responses are generated as:
        \[
        Y^{k+1} = X^{k+1}\hat{\theta}^k + \xi^{k+1}, \quad 
        \xi^{k+1}\sim \mathcal{N}(0,\sigma^2 I)
        \]
    \State \textbf{Verify:} Each synthetic sample $(x_i^{k+1}, y_i^{k+1})$ of $(X^{k+1}, Y^{k+1})$ is filtered by the verifier Condition~\eqref{equ:linear:verify_condition}, retaining only the verified subset $({X^{k+1}}', {Y^{k+1}}')$.
    \State \textbf{Retrain:} A new OLS estimator is computed using only the verified data:
        \[
        \hat{\theta}^{k+1} = ({{X^{k+1}}'}^\top {X^{k+1}}')^{-1} {{X^{k+1}}'}^\top {Y^{k+1}}'
        \]
\EndFor
\end{algorithmic}
\end{algorithm}

Given the orthogonality of $\{v_j\}$ in the block design~\eqref{equ:linear:block_design}, the OLS estimator decomposes into a set of one-dimensional problems, each estimating the coordinate of $\theta$ along direction $v_j$.
In particular, choosing $\{v_j\}$ as the right singular vectors of the real data matrix $X^0$ yields the cleanest interpretation, making explicit how verifier bias, synthetic sample size, and noise variance interact.
Accordingly, the retraining procedure can be formalized as follows:

\begin{algorithm}[H]
\caption{Iterative Verifier-Guided Retraining in Linear Regression}
\label{alg:linear_retraining}
\begin{algorithmic}[1]
    \State \textbf{Input:} Real data $(X^0, Y^0)$
    \State Compute initial estimator $\hat{\theta}^0 = ({X^{0}}^\top X^0)^{-1} {X^0}^\top Y^0$
    \State Let $X^0 = U \Sigma V^\top$ be the SVD of $X^0$, with right singular vectors $V = (v_1, \ldots, v_p)$
    \For{$k = 0, 1, 2, \ldots$}
        \For{$j = 1, \ldots, p$}
            \State Construct synthetic design matrix $X^{k+1,j}$ with all rows equal to $v_j^\top$
            \State Generate synthetic responses $Y^{k+1,j} = X^{k+1,j} \hat{\theta}^{k} + \xi^{k+1,j}$, where $\xi^{k+1,j} \sim \mathcal{N}(0, \sigma^2I)$
            \State Apply verifier to each $(x^{k+1,j}_i, y^{k+1,j}_i)$ and retain valid samples satisfying
            \begin{align}\label{equ:linear:verify_condition2}
            |y_i^{k+1,j} - (x_i^{k+1,j})^\top \theta_c| \leq r \|x_i^{k+1,j}\| + \sigma_c,    
            \end{align}
            \State yielding $n_k$ verified samples $(x'^{k+1,j}_i, y'^{k+1,j}_i)$.
            \State Compute one-dimensional estimator
            \begin{align}
            \hat{\theta}^{k+1,proj,j} =  \bar{y'}^{k+1,j}\label{equ:linear:thetak_proj}
            \end{align}
        \EndFor
        \State Update overall estimator:
        \begin{align}
        \hat{\theta}^{k+1} = \sum_{j=1}^p v_j \hat{\theta}^{k+1,proj,j} \label{equ:linear:thetak1}
        \end{align}
    \EndFor
\end{algorithmic}
\end{algorithm}

\begin{proof}[Proof of Theorem~\ref{thm:linear:one_step}]

We consider the one dimensional projection estimator of $\hat{\theta}^{1,proj,j}$ defined in \eqref{equ:linear:thetak_proj}.
The filter condition \eqref{equ:linear:verify_condition2} is equivalent to:
\begin{align}
     &|\sigma\xi^{1,j}_i + v_j^\top(\hat{\theta}^0 - \theta_c)| \leq r + \sigma_c\notag\\
\iff & y^{1,j}_i = \sigma\xi^{1,j}_i + v_j^\top\hat{\theta}^0 \in \left(-r-\frac{\sigma_c}{\sigma} + v_j^\top\theta_c,r + \frac{\sigma_c}{\sigma} + v_j^\top\theta_c\right).  \label{equ:linear:interval}
\end{align}
Note that $\hat{\theta}^0 \sim \mathcal{N}(\theta^\star,({X^{0}}^\top X^0)^{-1}\sigma^2)$ and $v_j$ is the $j$-th right singular vector of $X^0$, therefore $v_j^\top \hat{\theta}^0 \sim \mathcal{N}(v_j^\top \theta^\star, \sigma^2\mu^{-2}_j)$.
Therefore, $\hat{\theta}^{1,proj,j} = \bar{y}'^{1,j}$ correspond to the verifier-filtered mean estimator of a one-dimensional Gaussian mean estimation problem with true mean $v_j^\top \theta$, variance $\sigma^2\mu^{-2}_j$ and filtering interval $\left(-r-\frac{\sigma_c}{\sigma} + v_j^\top\theta_c,r + \frac{\sigma_c}{\sigma} + v_j^\top\theta_c\right)$.
Let
\begin{align}
     \alpha_j &:= \frac{-r-\sigma_c + v_j^\top(\theta_c-\theta^\star)}{\sigma}, \notag\\
     \beta_j &:= \frac{r+\sigma_c+ v_j^\top(\theta_c-\theta^\star)}{\sigma}. \label{equ:linear:alpha_beta_j}
\end{align}
Under the assumption $\mu_j = \omega(\sqrt{n_0})$, there exists a constant $L>0$, such that $\mu^2_j > L n_0$ for all $j=1, \ldots, p$.
Therefore, by Theorem~\ref{thm:gaussian:main}, there exists constant $K_j$ depending only on $\alpha_j,\beta_j$ such that if $n_1 > n_0 \geq 100$,
\begin{align}
     &\left| \frac{1}{\sigma^2}\mathbb{E}(\hat{\theta}^{1,proj,j}-v_j^{\top}\theta^\star)^2 - \frac{m_2(\alpha_j,\beta_j)}{n_1}- \bigg(m_1^2(\alpha_j,\beta_j) + \frac{m^2_2(\alpha_j,\beta_j)+m_3(\alpha_j,\beta_j)m_1(\alpha_j,\beta_j)}{\mu^2_j}\bigg) \right| \notag\\
     & < K_j\left(\frac{1}{n_1n^{1/3}_0}+\frac{1}{n^{3/2}_0}\right)\label{equ:linear:bound_j}
\end{align}
will hold with probability at least $1-\exp(-Ln_0^{1/3})$. $m_1,m_2,m_3$ are defined in \eqref{equ:trunc_normal_standard} and \eqref{equ:trunc_normal_standard2}.
By \eqref{equ:linear:thetak1}, we have $\hat{\theta}^{1,proj,j} =  v_j^{\top}\hat{\theta}^1$. 
In addition, since $V = (v_1,v_2, \ldots, v_p)$ is an orthonormal matrice, we have
\begin{align}
      \sum_{j=1}^p \mathbb{E}(\hat{\theta}^{1,proj,j} - v_j^{\top}\theta^\star)^2= \sum_{j=1}^p \mathbb{E}(v_j^{\top}\hat{\theta}^1 - v_j^{\top}\theta^\star)^2 = \mathbb{E} ||V^\top(\hat{\theta}^1 - \theta^\star)||^2 = \mathbb{E}\|\hat{\theta}^1 - \theta^\star|^2.
\end{align}
Therefore, by summing over $j$ on both sides of \eqref{equ:linear:bound_j} and using simple union bound, we have
\begin{align}
 &\left|\frac{1}{\sigma^2}\mathbb{E}||\hat{\theta}^1 - \theta^\star||^2 - \sum_{j=1}^p \left(\underbrace{\frac{m_{2,j}}{n_1}}_{{\text{Synthetic Variance}}} + \underbrace{m^2_{1,j}+\frac{m_{1,j}m_{3,j}+m^2_{2,j}}{\mu^2_j}}_{{\text{Verification Error}}}\right)\right| < K\left(\frac{1}{n_1n^{1/3}_0}+\frac{1}{n^{3/2}_0}\right)
 \end{align}
 with $K = \max_{j}{K_j}$ and
\begin{align}     
     m_{1,j} &:= m_1(\alpha_j,\beta_j),\notag\\
     m_{2,j} &:= m_2(\alpha_j,\beta_j),\notag\\
     m_{3,j} &:= m_3(\alpha_j,\beta_j).\notag
\end{align}

\end{proof}

\begin{proof}[Proof of Theorem~\ref{thm:linear:long_term}]
     
We consider the transition dynamics of $\hat{\theta}^k$ in Algorithm \ref{alg:linear_retraining}.
Since we designed $X^{k+1,j}$ to be the rank one matrix correspond to singular vector $v_j$, therefore \eqref{equ:linear:thetak_proj} can be rewritten as
\begin{align}
    \hat{\theta}^{k+1,proj,j} = v_j^\top \hat{\theta}^k + \frac{\sigma}{n_k} \sum_{i=1}^{n_k} \xi'^{k+1,j}_i \label{equ:linear:thetakj2}
\end{align}
where $\xi'^{k+1,j}_i$ is the truncated noise term after verification. By \eqref{equ:linear:verify_condition2}, we have

\begin{align}
\xi'^{k+1,j}_i \; \text{i.i.d} \sim \mathcal{N}_{trunc}\left(-\frac{r}{\sigma}-\frac{\sigma_c}{\sigma} -v^{\top}_j\frac{\hat{\theta}^k - \theta_c}{\sigma}, \frac{r}{\sigma}+\frac{\sigma_c}{\sigma} -v^{\top}_j\frac{\hat{\theta}^k - \theta_c}{\sigma}\right).
\end{align}

We consider the rotated standardized estimator 
\begin{align}
\epsilon^k_j := v^{\top}_j \frac{\hat{\theta}^k-\theta_c}{\sigma} \quad \text{equivalently} \quad \epsilon^k := V^\top \frac{\hat{\theta}^k-\theta_c}{\sigma}.\notag
\end{align}
Since $\hat{\theta}^{k+1,proj,j} = v_j^{\top}\hat{\theta}^{k+1}$ by \eqref{equ:linear:thetak1}, \eqref{equ:linear:thetakj2} can be standardized as
\begin{align}
    \epsilon^{k+1}_j = \epsilon^k_j + \frac{\sum_{i=1}^{n_k} \xi'^{k+1,j}_i}{n_k}, \qquad \xi'^{k+1,j}_i \; \text{i.i.d} \sim \mathcal{N}_{trunc}\left(-\beta - \epsilon^k_j, \beta - \epsilon^k_j\right)\label{equ:linear:transition_standard}
\end{align}
where $\beta = \frac{r}{\sigma} + \frac{\sigma_c}{\sigma}$. We note that \eqref{equ:linear:transition_standard} is exactly the same dynamics we consider in the proof of Theorem~\ref{thm:gaussian:limit} with $\beta = -\alpha < \infty$.
In  other words, the evolution of the iterative estimator $\epsilon^k$ is diagonal and each cordinates follows the same dynamics as the one dimensional gaussian iterative mean estimator.
From Theorem~\ref{thm:gaussian:limit}, we known that there exists a constant $\rho < 1$ such that

\begin{align}
    \mathbb{E} \|\epsilon^k_j\|^2 \leq \rho^{2k} \mathbb{E}\|\epsilon^0_j\|^2 + \sum_{j=0}^{k-1}\frac{\rho^{2(k-j)-1}}{n_j}, \qquad j = 1,2,\ldots,p.\notag
\end{align}
This implies that
\begin{align}
    \mathbb{E} \|\hat{\theta}^k-\theta_c\|^2 \leq \rho^{2k} \mathbb{E}\|\hat{\theta}^0-\theta_c\|^2 + p\sigma^2\sum_{j=0}^{k-1}\frac{\rho^{2(k-j)-1}}{n_j}.\notag
\end{align}     

\end{proof}

\clearpage
\section{Additional Details on CVAE Experiments}\label{app:detail}

\paragraph{Data preprocessing.}
We use MNIST ($28\times 28$ grayscale) and normalize pixel intensities to $[0,1]$. 
Class labels are represented as one-hot vectors $y\in\{0,1\}^{K}$ ($K{=}10$). 

\paragraph{Experiment Details.} 
We use a convolutional CVAE model consisting of an Encoder with two convolutional layers 
(1$\to$32 and 32$\to$64 channels, $4\times4$ kernels, stride 2, with GELU activations), 
followed by a linear projection that outputs the mean and log-variance of a $d_z=20$-dimensional 
Gaussian latent space. 
The Decoder mirrors this structure: a linear layer maps the latent code to a 
$64\times7\times7$ tensor, which is upsampled by two transposed convolutional layers 
(64$\to$32 and 32$\to$1 channels, $4\times4$ kernels, stride 2, with GELU activations) 
to reconstruct $28\times28$ images. 
We train the CVAE with the standard objective, i.e., binary cross-entropy reconstruction loss 
plus KL divergence regularization.

\paragraph{Discriminator for filtering.} 
We additionally train a discriminator $D$ to distinguish real from synthetic samples. 
$D$ is implemented as a multi-layer perceptron: five fully connected layers with hidden sizes 
512, 256, 128, and 64, each followed by a LeakyReLU activation, 
and a final linear layer mapping to a single logit. 
The output is passed through a sigmoid to yield the probability of the input being real. 
The discriminator is trained with binary cross-entropy, labeling real MNIST digits as positive 
and CVAE-generated digits as negative. 

\paragraph{Synthetic generation and filtering.}
After each training round, we generate conditioned samples by drawing 
$z \sim \mathcal{N}(0,I)$, choosing labels $y$ (uniform over classes unless specified), 
and decoding $\tilde{x} = g_{\theta}(z,y)$. 
To control sample quality, we score each $(\tilde{x},y)$ with the discriminator $D(\tilde{x},y)$. 
For each class, we retain only the top $10\%$ of generated samples with the highest discriminator 
scores. These filtered synthetic samples form the 
training data for the next round.

\paragraph{Supplementary Results on ELBO}
\textcolor{black}{
We also evaluate generative performance using the test negative ELBO, a standard likelihood-based loss metric for VAEs. 
To prevent overfitting the discriminator (i.e., the verifier) and ensure stability during the retraining cycles, we incorporate standard regularization techniques, specifically applying a dropout rate of 0.1 and label smoothing with a parameter of 0.05 when training the discriminator.
To investigate the effect of synthetic data size $n_k$, we employ three linearly increasing sample size schedules.
Starting with an initial CVAE trained on only 500 samples, we scale up the retraining size by adding 5K, 30K, or 50K synthetic samples per iteration, respectively.
The models are retrained for 50 iterations until the test negative ELBO stabilizes.}

Figure~\ref{fig:iter_recon_loss} reports the test negative ELBO over these 50 rounds.
Consistent with our bias--variance analysis, we observe a clear improvement (a decrease in loss) in the early stages (up to roughly iterations 10--15).
Furthermore, the trajectories reveal a critical dynamic: while larger synthetic size schedules significantly accelerate this early convergence, all three schedules ultimately plateau and converge to a similar negative ELBO value by iteration 50. 
This observation validates our theoretical framework: drawing more synthetic samples expedites the initial variance reduction phase, but the asymptotic performance limit is dictated by the verifier, not the volume of synthetic data.
After the initial variance-reduction gains, the negative ELBO eventually reverses its trend and increases (deteriorates) as the model inevitably converges toward the verifier's knowledge center.

As discussed in the main text regarding our verifier's limitations, this knowledge center is demonstrably biased.
For reference, across all three size schedules, the final retrained CVAEs at iteration 50 converge to a test negative ELBO of approximately 111. In contrast, a baseline CVAE trained on the entire 60K real image dataset attains a test negative ELBO of 92.12 (lower is better).
Because our verifier emphasizes perceptual quality over likelihood-based reconstruction, the negative ELBO proves harder to improve than FID. 
As a result, even as the negative ELBO stagnates or worsens in later iterations, our retrained models continue to improve FID, achieving sharper, cleaner digits.
We believe that deploying stronger verifiers with, e.g., diversity preservation capabilities could enable iterative retraining to further improve the negative ELBO.

\begin{figure}[H]
    \centering
    \includegraphics[width=0.6\textwidth]{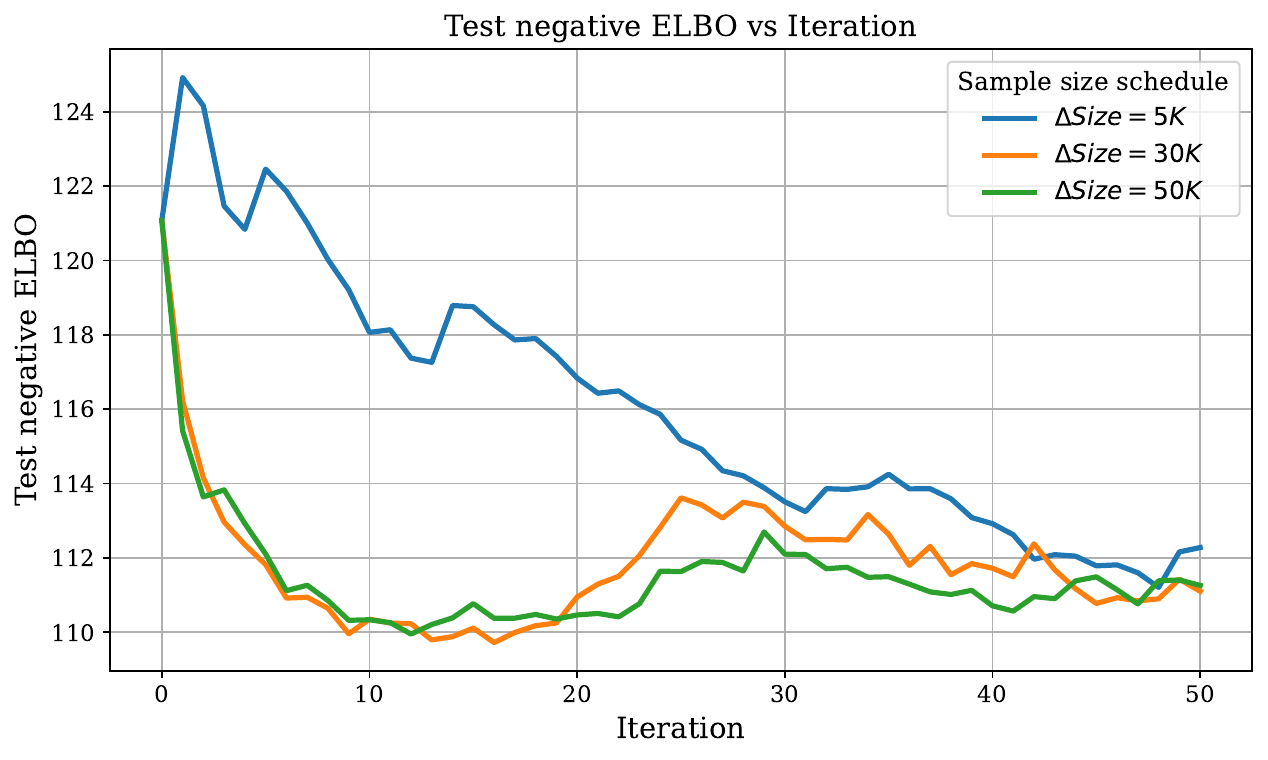}
    \caption{Test negative ELBO across retraining iterations.}
    \label{fig:iter_recon_loss}
\end{figure}

\section{Additional Experimental Results}\label{app:addexp}

\subsection{Random Synthetic Data in Linear Regression}\label{sec:random_synthetic_data}

In the main text, the synthetic covariates were aligned with a fixed orthonormal basis to
simplify analysis and make the retraining dynamics easier to interpret. To show that the
observed behavior is not tied to this structured design, we repeat the same iterative
retraining experiment using fully random synthetic covariates sampled i.i.d.\ from a
standard Gaussian distribution.

Figure~\ref{fig:linear_random} presents the results, corresponding directly to the lower two
panels in Figure~\ref{fig:linear_combined} of the main text, but under the random-design
setting. The qualitative behavior remains the same: with a well-specified verifier,
retraining contracts toward the verifier's knowledge center and avoids collapse, whereas
unfiltered retraining diverges. This confirms that the verifier-induced stability and
improvement patterns hold beyond the orthonormal-design assumption.

\begin{figure}[H]
    \centering
    \begin{subfigure}[b]{0.495\textwidth}
        \centering
        \includegraphics[width=\linewidth]{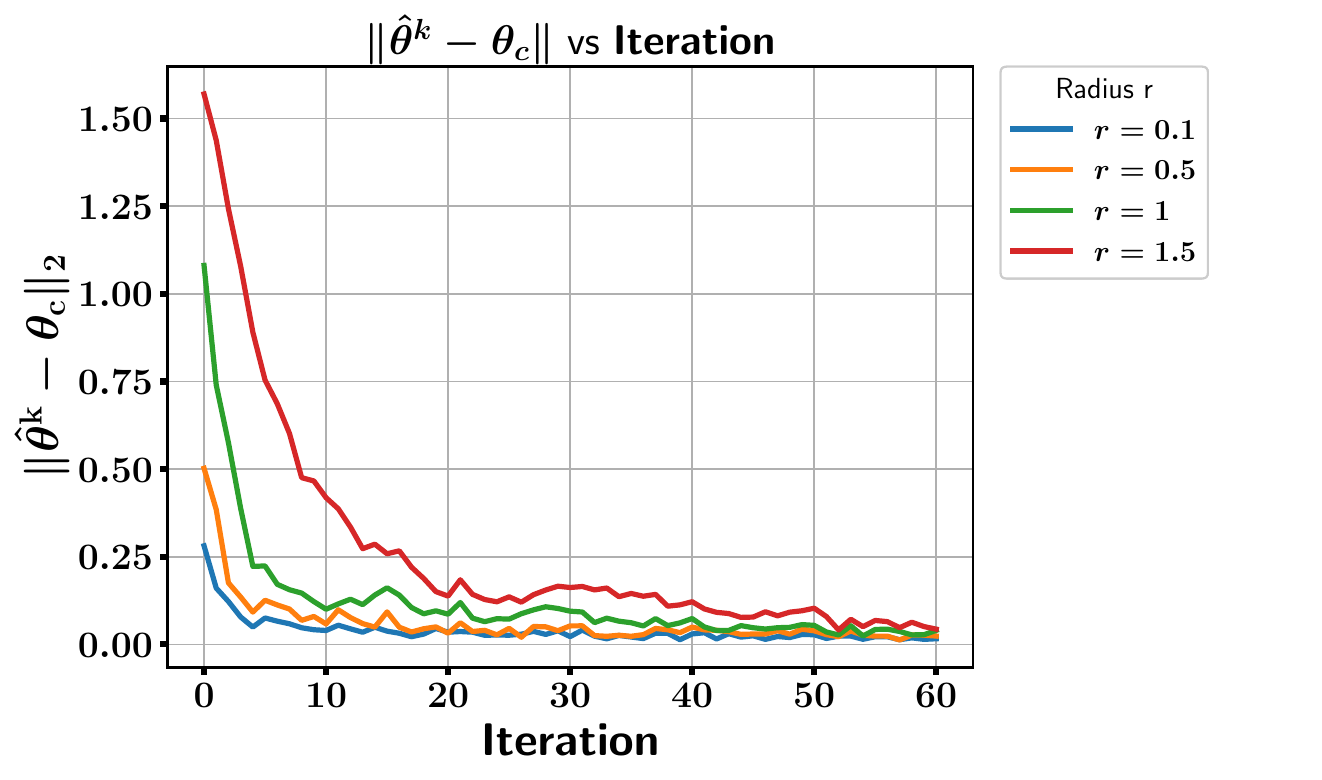}
        \caption{Verifier with bias (random design)}
    \end{subfigure}
    \hfill
    \begin{subfigure}[b]{0.495\textwidth}
        \centering
        \includegraphics[width=\linewidth]{./plot/linear1_random_2.pdf}
        \caption{Verifier without bias (random design)}
    \end{subfigure}
    \caption{{Iterative synthetic retraining under random synthetic covariates, 
    corresponding to the structured-design results in Figure~\ref{fig:linear_combined}.}}
    \label{fig:linear_random}
\end{figure}
\subsection{Different Verifier Shapes}\label{sec:verifier_shape}

We further analyze how different geometric choices of the verifier region affect the
acceptance rule and the resulting retraining dynamics.
For any region $\mathcal{R}_\theta$ around a center $\theta_c$, a synthetic point
$(x,y)$ is accepted whenever there exists a parameter perturbation $\Delta$ in the
region that can explain $y$, i.e.
\[
   y = x^\top (\theta_c + \Delta) + \xi, \qquad \Delta \in \mathcal{R}_\theta.
\]
This leads to the general acceptance requirement
\[
   |y - x^\top \theta_c|
   \le \sup_{\Delta \in \mathcal{R}_\theta} |x^\top \Delta| + \sigma_c.
\]
Different verifier shapes correspond to different support functions 
$\sup_{\Delta \in \mathcal{R}_\theta} |x^\top \Delta|$.

\paragraph{(1) Ellipsoidal verifier.}
Consider the anisotropic ellipsoid
\[
    \mathcal{R}_\theta = \bigl\{
        \theta : (\theta - \theta_c)^\top A(\theta - \theta_c) \le r^2
    \bigr\}, \qquad A \succ 0.
\]
Let $\Delta = \theta - \theta_c$.  
Changing variables $\Delta = A^{-1/2}u$ with $\|u\|_2 \le r$ yields
\[
    \sup_{\Delta^\top A \Delta \le r^2} |x^\top \Delta|
    = r\|A^{-1/2}x\|_2
    = r\sqrt{x^\top A^{-1}x}.
\]
Thus the acceptance condition becomes
\[
    |y - x^\top \theta_c|
    \le r\sqrt{x^\top A^{-1}x} + \sigma_c.
\]

\paragraph{(2) Polyhedral $\ell_1$ verifier.}
For the $\ell_1$ knowledge region
\[
    \mathcal{R}_\theta = \{\|\theta - \theta_c\|_1 \le r\},
\]
the perturbation satisfies $\|\Delta\|_1 \le r$.  
Using Hölder duality,
\[
    \sup_{\|\Delta\|_1 \le r} |x^\top \Delta|
    = r\|x\|_\infty.
\]
The corresponding acceptance rule is
\[
    |y - x^\top \theta_c|
    \le r\|x\|_\infty + \sigma_c.
\]

Although ellipsoidal and $\ell_1$ (polyhedral) regions induce different forms of 
acceptance sets, both yield the same qualitative retraining behavior:
\emph{$\hat{\theta}^{(k)}$ consistently move toward the verifier center $\theta_c$}. The empirical trajectories under both shapes are shown in Figure~\ref{fig:verifier-shapes}.

\begin{figure}[H]
    \centering

    \begin{subfigure}[b]{0.495\textwidth}
        \centering
        \includegraphics[width=\linewidth]{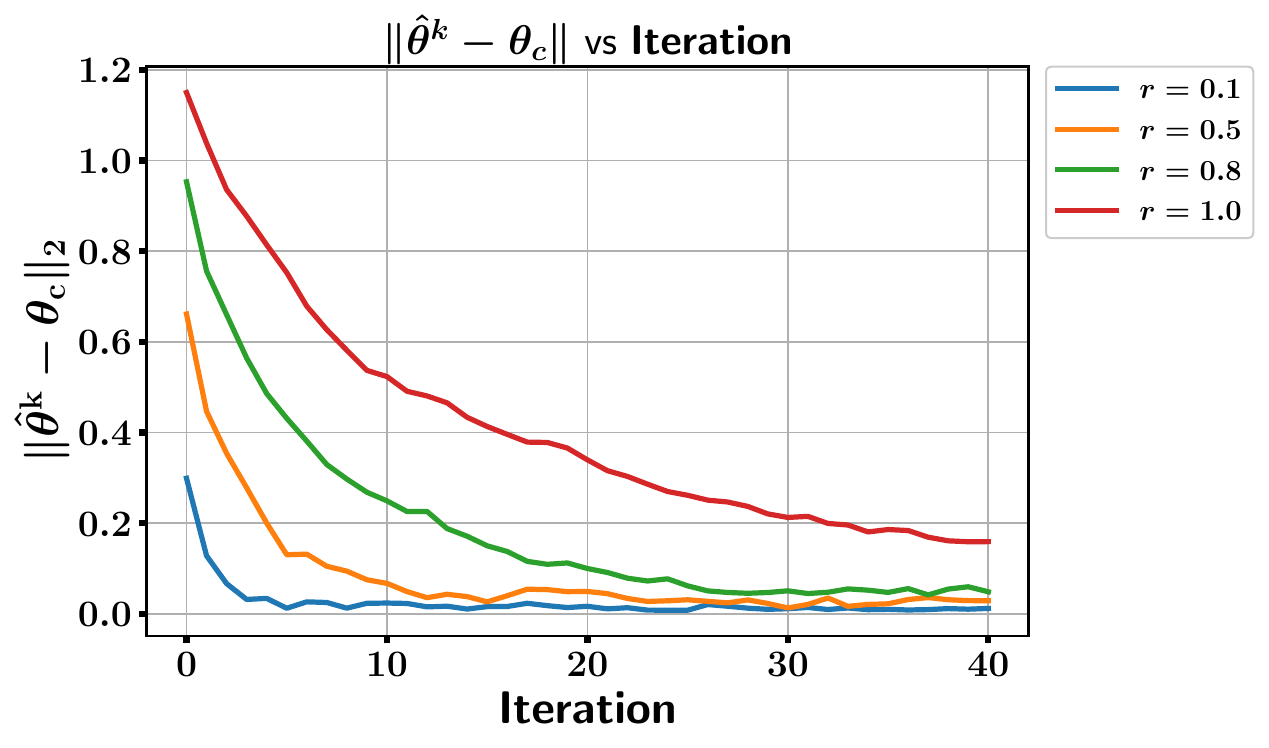}
        \caption{Ellipsoidal verifier}
    \end{subfigure}
    \hfill
    \begin{subfigure}[b]{0.495\textwidth}
        \centering
        \includegraphics[width=\linewidth]{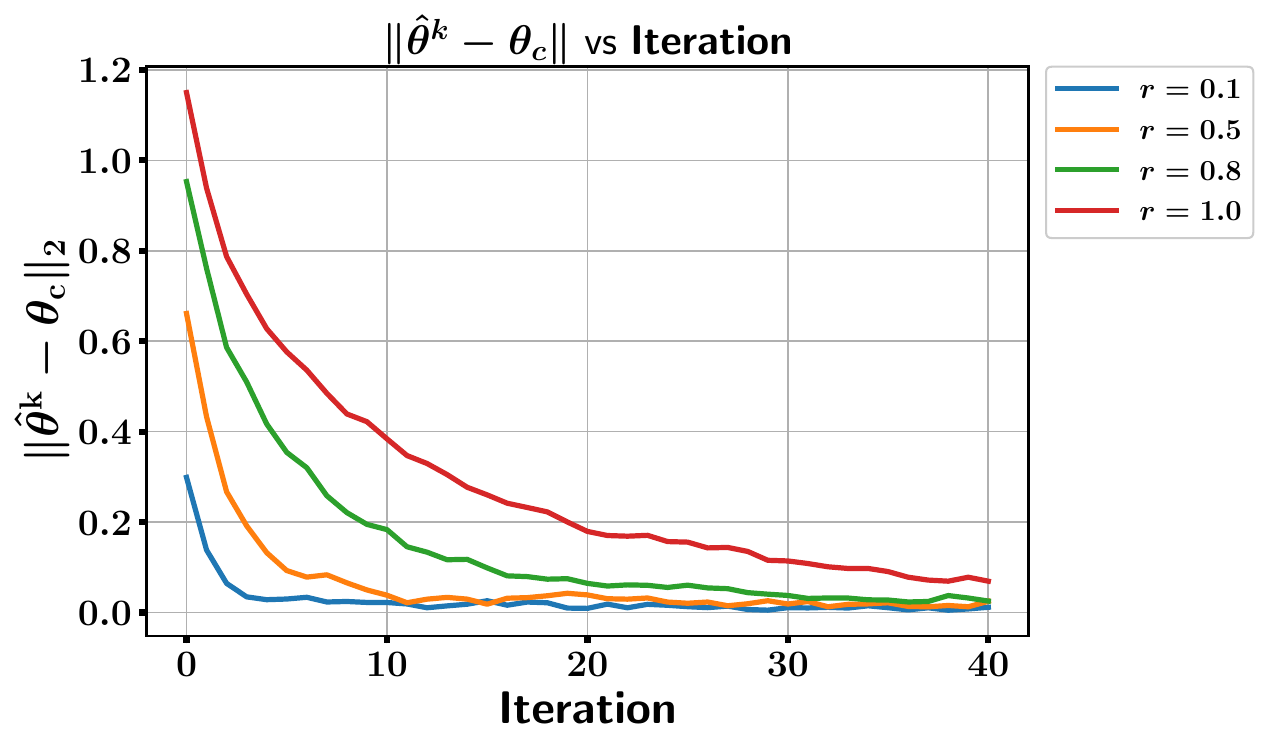}
        \caption{Polyhedral $\ell_1$ verifier}
    \end{subfigure}

    \caption{Retraining trajectories under two different verifier shapes.  
    In both cases, $\hat{\theta}^{(k)}$ empirically converges toward the 
    verifier center $\theta_c$.}
    \label{fig:verifier-shapes}
\end{figure}

\subsection{MNIST-Specific FID Evaluation}
\label{sec:mnist-fid}
The standard Fréchet Inception Distance (FID) is widely used in generative modeling, including on MNIST, following prior work such as \citet{daidiagnosing,leontev2020quality,chan2024evaluating}.  Nonetheless, we agree that Inception embeddings are not tailored to handwritten digits and may not fully capture perceptual similarity on MNIST.

To address this point, we introduce a \textbf{MNIST-specific FID} variant.  
We train a lightweight convolutional network directly on MNIST classification, and compute FID using the penultimate-layer activations as the embedding space.  
This produces a domain-appropriate FID measure while preserving the same statistical structure as the original metric. These results confirm that our conclusions are robust to the choice of embedding and do not depend on the use of vanilla FID.

\vspace{0.5em}
\noindent\textbf{Results.}
Figures~\ref{fig:mnist-fid-v2-full} and \ref{fig:mnist-fid-v2-zoom} report the new FID scores under our retraining framework for all verifier sizes.  
Consistent with the standard FID curves in the main paper.
\begin{figure}[H]
    \centering
    \begin{subfigure}[t]{0.495\textwidth}
        \centering
        \includegraphics[width=\linewidth]{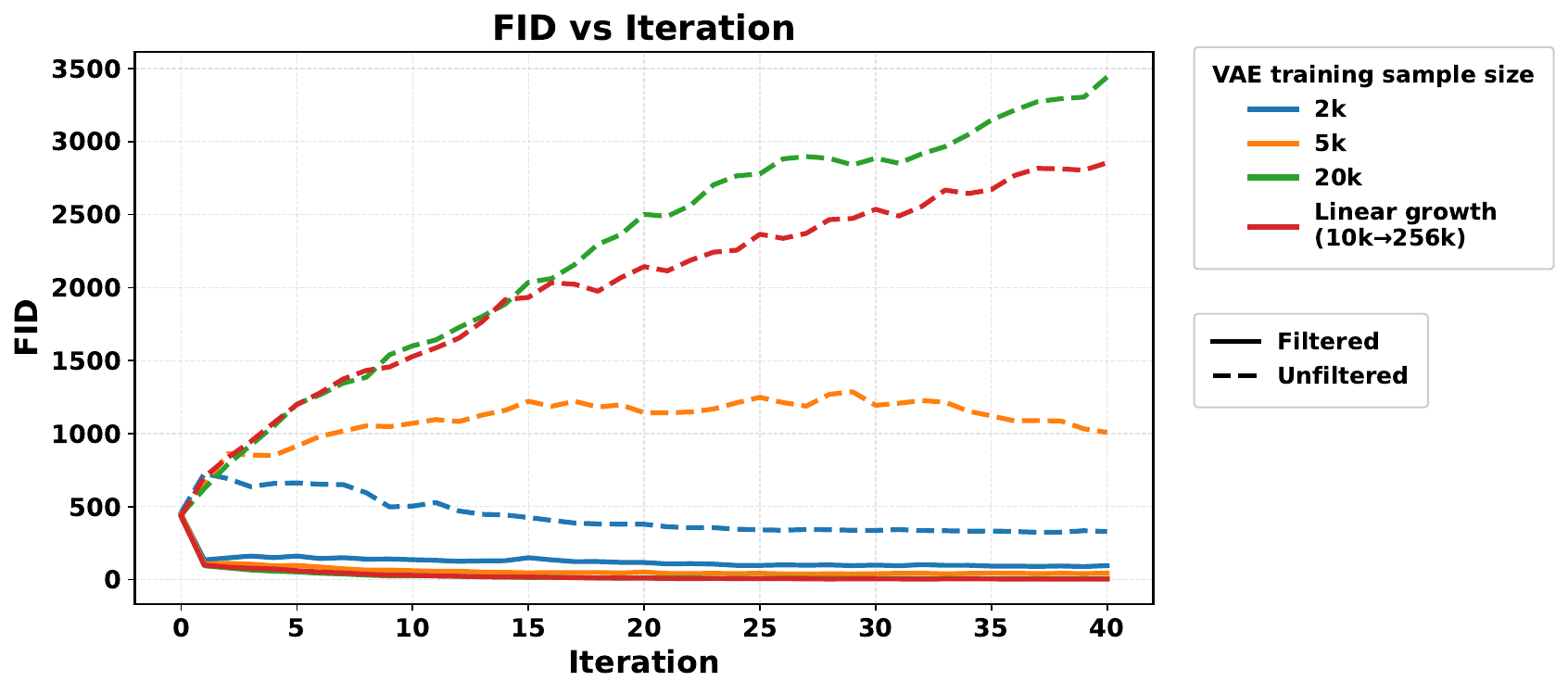}
        \caption{\textbf{MNIST-specific FID over retraining iterations.}}
        \label{fig:mnist-fid-v2-full}
    \end{subfigure}
    \hfill
    \begin{subfigure}[t]{0.495\textwidth}
        \centering
        \includegraphics[width=\linewidth]{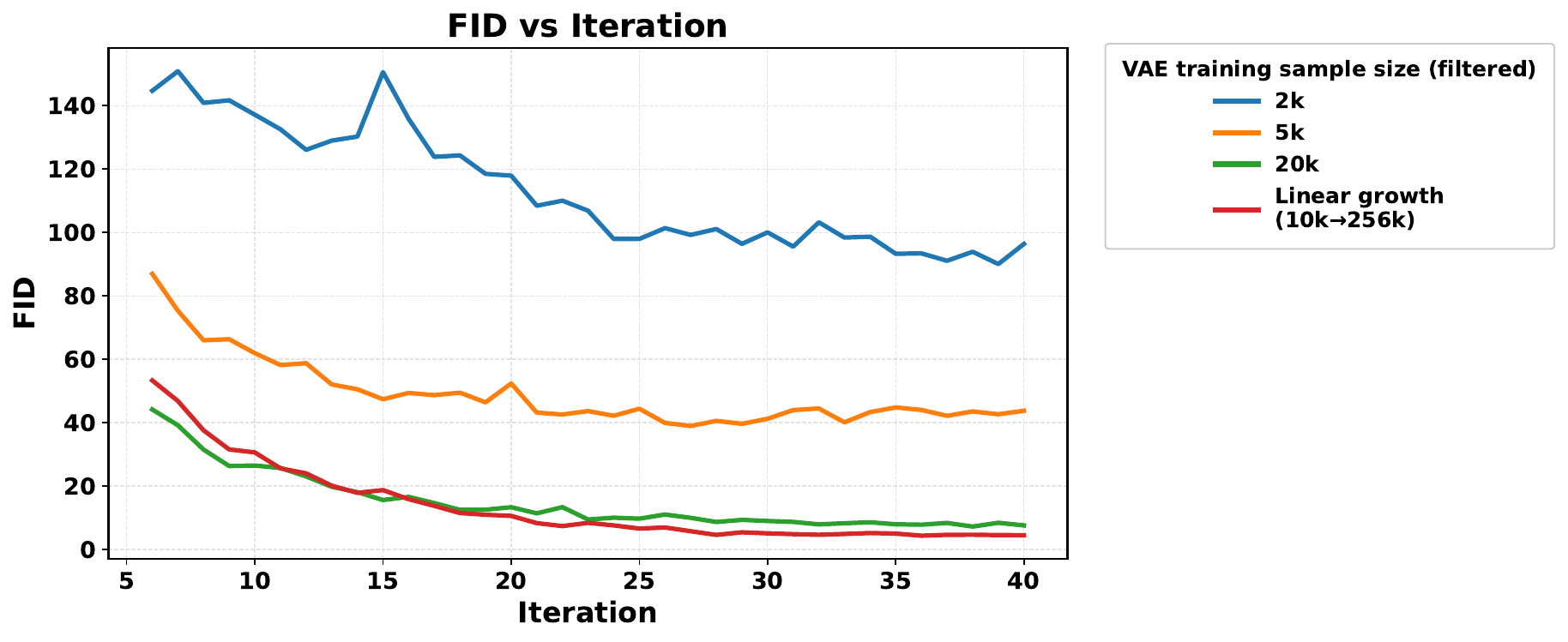}
        \caption{\textbf{Zoom-in view for filtered runs.}}
        \label{fig:mnist-fid-v2-zoom}
    \end{subfigure}
    \caption{\textbf{MNIST-specific FID using our MNIST-trained feature embedding.}  
    Both results confirm that our conclusions remain unchanged when replacing standard FID with a domain-specific metric.}
    \label{fig:mnist-fid-v2-both}
\end{figure}

\subsection{Different Initial Sample Sizes}\label{sec:initial_sample_size}

We assess the robustness of verifier-guided retraining by varying the number of real
MNIST images used to train the initial CVAE ($1k, 2k, 3k, 4k, 60k$). For small and
medium initial sample sizes, verifier filtering yields clear early FID improvements and then stabilizes performance, whereas unfiltered retraining quickly degrades. When the
generator is initialized on all 60k real images, verifier filtering no longer improves FID over the initial model, but it still effectively prevents the severe collapse observed under
unfiltered retraining.

We perform our main experiments in a low–real-data regime (e.g., with 500 initial images), where the verifier, having been trained on a much larger subset of MNIST, holds strictly more external knowledge than the generator. According to our theory, this is exactly the regime in which verifier-guided retraining should provide true improvement rather than simple stabilization, because the verifier contributes additional information. In contrast, when the generator is initialized on the full 60k training images,
the verifier would need access to an even stronger source of external information to
achieve improvement; otherwise it can only prevent collapse. For this reason, the small
initial sample size serves as the most informative regime for highlighting the verifier’s
knowledge-injection effect and demonstrating the improvement phenomena predicted
by our theoretical framework.

\begin{figure}[H]
    \centering
    \includegraphics[width=0.55\linewidth]{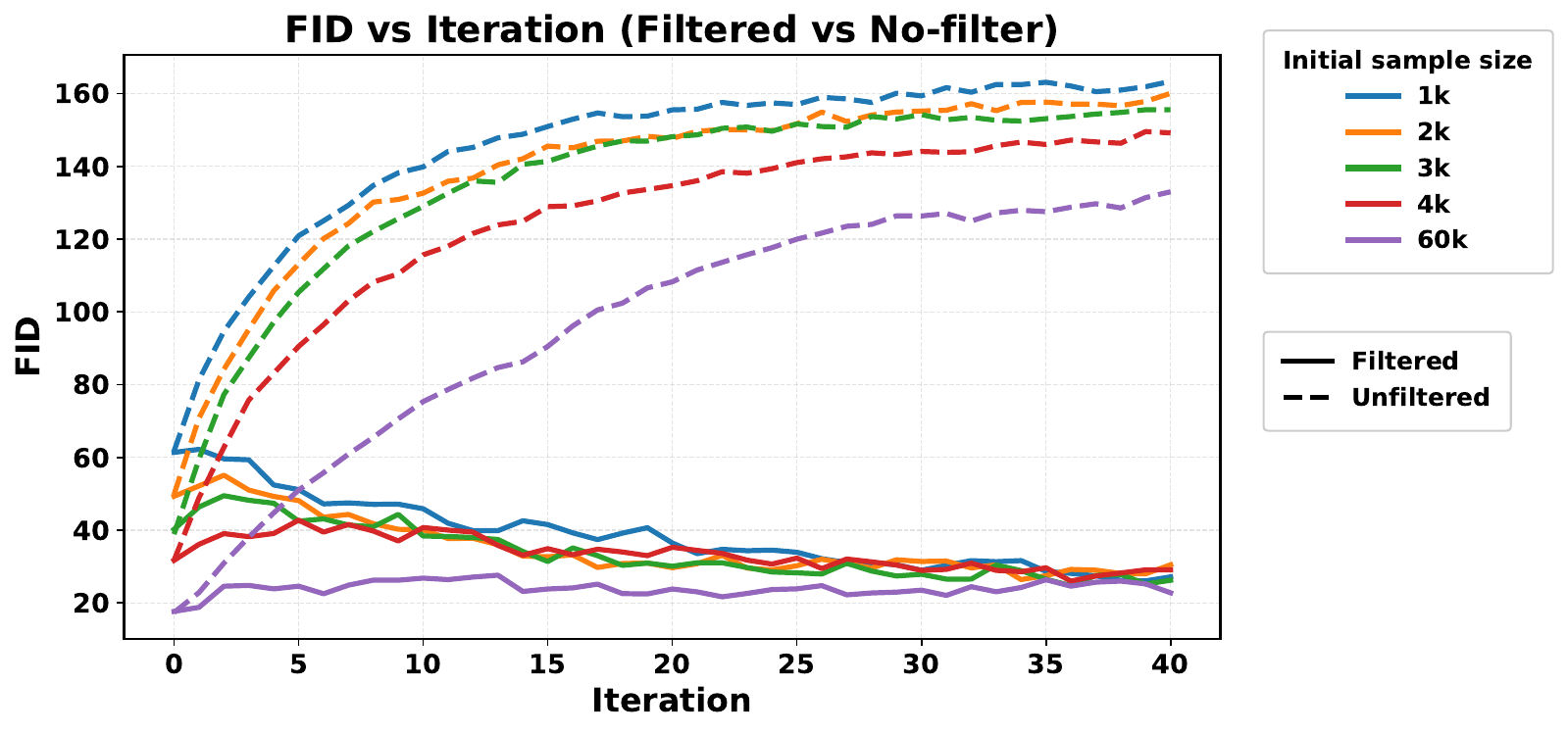}
    \caption{FID across retraining iterations under different initial sample sizes, 
    comparing verifier-filtered and unfiltered synthetic retraining.}
    \label{fig:placeholder}
\end{figure}

\end{document}